\documentclass{article}

\usepackage[preprint]{neurips_2020}
\usepackage[utf8]{inputenc} 
\usepackage[T1]{fontenc}    
\usepackage{hyperref}       
\usepackage{url}            
\usepackage{booktabs}       
\usepackage{amsfonts}       
\usepackage{nicefrac}       
\usepackage{microtype}      
\usepackage{bm}
\usepackage{amsmath}
\usepackage{amssymb}
\usepackage{amsthm}
\usepackage{amsmath}
\usepackage{algorithm}
\usepackage[noend]{algpseudocode}
\usepackage{graphicx}
\usepackage{subfigure}
\usepackage{enumitem}
\usepackage{diagbox}
\usepackage{makecell}

\newtheorem{theorem}{Theorem}

\newtheorem{proposition}{Proposition}
\newtheorem{definition}{Definition}
\newtheorem{example}{Example}
\newtheorem{remark}{Remark}
\newtheorem{assumption}{Assumption}
\newcommand{\PP}{\mathbb{P}}
\newcommand{\EE}{\mathbb{E}}
\newcommand{\EEE}{\mathcal{E}}
\newcommand{\VV}{\mathbb{V}}
\newcommand{\FF}{\mathcal{F}}
\newcommand{\CC}{\mathcal{C}}
\newcommand{\RR}{\mathcal{R}}

\newcommand{\GG}{\mathcal{G}}

\title{Risk-Sensitive Reinforcement Learning: a Martingale Approach to Reward Uncertainty}

\author{
  Nelson Vadori, Sumitra Ganesh, Prashant Reddy, Manuela Veloso  \\
  JP Morgan AI Research\\
  \footnotesize
  \texttt{\{nelson.n.vadori, sumitra.ganesh, prashant.reddy, manuela.veloso\}@jpmorgan.com}\\ 
  \normalsize
}
\begin{document}

\maketitle

\begin{abstract}
We introduce a novel framework to account for sensitivity to rewards uncertainty in sequential decision-making problems. While risk-sensitive formulations for Markov decision processes studied so far focus on the distribution of the cumulative reward as a whole, we aim at learning  policies sensitive to the uncertain/stochastic nature of the rewards, which has the advantage of being conceptually more meaningful in some cases. To this end, we present a new decomposition of the randomness contained in the cumulative reward based on the Doob decomposition of a stochastic process, and introduce a new conceptual tool - the \textit{chaotic variation} - which can rigorously be interpreted as the risk measure of the martingale component associated to the cumulative reward process. We innovate on the reinforcement learning side by incorporating this new risk-sensitive approach into model-free algorithms, both policy gradient and value function based, and illustrate its relevance on grid world and portfolio optimization problems. 
\end{abstract}

\section{Introduction}
\label{intro}
Classical reinforcement learning (RL) aims at deriving policies that maximize the expected value of the sum (or average) of future rewards, while so-called \textit{risk-sensitive} RL aims at taking into account some measure of variability of the cumulative reward into the learned policy, usually through a risk criterion such as variance or a risk measure (e.g. entropic, CVaR). The common goal to the existing risk-sensitive RL literature (cf. section \ref{relwork}) is to take into account the distribution of the cumulative rewards in order to learn a variety of policies, usually parametrized by a risk parameter such as the mean-variance trade-off, the CVaR percentile or an upper bound on variance. For example, in the mean-variance case, learned policies typically lead to the distribution of cumulative rewards having lower mean but also lower variance (hence we gain confidence on the outcome at the expense of its mean).

 In the existing literature, the chosen risk criterion is applied to the cumulative reward as a whole, and hence does not distinguish between the different sources of randomness contained in it. The motivating example of section \ref{toy} and the Doob decomposition of section \ref{Doob} show that the randomness contained in the cumulative reward actually consists of two components of different nature and having different practical interpretations. If we denote $R_\pi(s_t,s_{t+1}):=R(s_t,\pi(s_t),s_{t+1})$ the reward obtained at time $t+1$ associated to policy $\pi$, the first \textit{chaotic} component exactly captures the non-deterministic nature of the reward, i.e. that it is uncertain (or stochastic) when action $a_t=\pi(s_t)$ is chosen in state $s_t$ (as it depends on $s_{t+1}$). This component is zero if the reward is deterministic. The second \textit{predictable} component replaces the uncertain rewards by their predictable/deterministic projections $\EE[R_\pi(s_t,s_{t+1})|s_t,\pi(s_t)]$, and hence does not depend on their uncertainty/stochasticity but acts as if the reward we get at time $t+1$ after choosing action $\pi(s_t)$ is always what we predicted it to be based on information available at time $t$. This component possesses some variability due to the switching between states, but no due to the uncertain nature of the rewards. Interestingly, in the average reward case, these two components also appear in the limiting process of the central limit theorem for functionals of Markov chains, applied to the average reward $\frac{1}{n}\sum_{t=0}^n R_\pi(s_t,s_{t+1})$. Indeed, the limit of this term, minus its long-range (ergodic) mean and rescaled by $\sqrt{n}$, converges in distribution as $n \to +\infty$ to the sum of two independent normal random variables with respective variances $\sigma_{deter.}^2$ and $\sigma_{chaotic}^2$, where the latter depends on the uncertainty of the rewards, whereas the former only depends on their deterministic projections $\EE[R_\pi(s_t,s_{t+1})|s_t,\pi(s_t)]$ as discussed above. We refer to the supplementary for more details on this observation.

 In this work, we develop a new approach complementary to the existing literature that learns policies sensitive to the variability contained in the chaotic component reflecting the uncertainty/stochasticity of the rewards, discussed above. We will work in a framework where the reward $R(s_t,a_t,s_{t+1},h_{t+1})$ received at $t+1$ after having taken action $a_t$ is \textit{uncertain} in the sense that it depends on the next state $s_{t+1}$ and/or possibly on extra "hidden" information $h_{t+1}$ (cf. the setting of \cite{Shen2013-ds} for example).
 
 We believe that this approach is of interest for the following reasons. First, as mentioned in \cite{Chow2018-eq}, deriving risk criteria that are conceptually meaningful is a topic of current research. In our case, sensitivity to the chaotic reward component can nicely be interpreted as sensitivity to the \textit{uncertainty of the future reward when an action $a_t$ is chosen}, in other words to the difference between the actual reward received at $t+1$ and our best guess of it given information available at time $t$, which is a natural human behavior that is adopted in many real-world situations. For example in section \ref{appl}, we show on a portfolio optimization problem how classical mean-variance based risk-sensitive RL counterintuitively leads to reduce investment in both the risky and the risk-free asset as the risk aversion parameter increases (and hence we stop taking advantage of the risk-free asset), whereas our new chaotic mean-variance algorithm leads to reduce investment in the risky asset \textit{only}, which is intuitive as the evolution of the latter is unknown when the investment decision is made and hence can potentially yield significant losses.
 
 Second, depending on the application in mind, mixing the above mentioned predictable and chaotic sources of randomness together and treating it as a single "noise" may lead to undesirable and counterintuitive learned policies, and a more subtle understanding of the randomness is needed in order to gain more interpretability in the learned risk-sensitive policies (cf. section \ref{toy}). It was mentioned in \cite{Mannor2011-vl} that variance as a risk criterion may sometimes lead to counterintuitive policies, in that an agent who has received unexpected large rewards may seek to incur losses in order to keep the variance small, since variance penalizes deviations of the cumulative reward from its expected value. Hence, \cite{Prashanth2016-fz} have considered per-period variance instead. An interesting observation is that when applied to the chaotic component only (which will prove to be a martingale), these two notions of overall and per-period variance bridge under the concept of quadratic variation $\left<M\right>$ of a martingale $M$, due to the equality $\EE[M^2]=\EE[\left<M\right>]$ (proposition \ref{cve}): penalizing reward uncertainty/stochasticity will not force the cumulative reward towards some baseline level and hence will not lead to counterintuitive policies mentioned above.
 
\subsection{Main contributions}
We provide in section \ref{Doob} a novel, conceptually meaningful decomposition of the cumulative reward process that distinguishes between the different sources of randomness contained within it (\textit{chaotic} and \textit{predictable}, cf. discussion above). The key idea of the paper is to apply the Doob decomposition of a stochastic process to the cumulative reward process. We believe that this decomposition in itself is of interest in gaining a better understanding of the variability contained within the reward process.

We then introduce in section \ref{chaotic} a new definition of risk that exactly captures reward uncertainty risk, i.e. the risk related to its non-deterministic nature: the \textit{chaotic variation} associated to the reward process (and to a risk measure).

We incorporate for the first time reward uncertainty risk into model-free value-function based and policy gradient algorithms: the related modified versions of some canonical RL algorithms, so-called \textit{Chaotic Mean-Variance algorithms} (CMV), are presented in section \ref{rlchaos}. These algorithms are applied to grid world and portfolio optimization problems in section \ref{appl}. Although the focus in the main text is on the chaotic mean-variance case, our analysis in section \ref{chaotic} allows for general risk measures and we discuss some of them in the supplementary, such as the chaotic CVaR and Sharpe ratio.

\subsection{Motivating example}
\label{toy}
 As mentioned above, we are interested in developing a general RL framework to account for reward uncertainty in learned policies. The example presented in this section aims at illustrating what could sometimes go wrong when taking variance as a measure of risk, as it is commonly the case in the risk-sensitive RL literature (cf. section \ref{relwork}). We consider the episodic toy example of $T$ timesteps where there are 2 states, 2 actions, and at each state transition the probability to reach either states is the same, i.e. $P(s_{0}=2)=P(s_{0}=1)=P(s_{t+1}=1|s_t,a_t)=P(s_{t+1}=2|s_t,a_t)=\frac{1}{2}$. If $s_t=1$, the reward obtained at $t+1$ is 2 if $a_t=1$ and $4+\sigma h_{t+1}$ if $a_t=2$; if $s_t=2$, the rewards obtained at $t+1$ are 10 if $a_t=1$ and $8+\sigma h_{t+1}$ if $a_t=2$, where $h_{t}$ are i.i.d. zero mean unit variance, and the noise $\sigma \geq 0$. Think for example as action $a_t=2$ as investing in a risky stock at time $t$ and getting the corresponding price change at time $t+1$. If $\pi_i$ is the policy that always selects action $i$, then $\pi_1, \pi_2$ have the same expected cumulative reward of $6T$. We have plotted in figure \ref{fig-toy} the rewards obtained for the 2 policies in each state. 

 It can be proved that if the noise $\sigma^2$ is small enough (cf. supplementary for a precise quantitative statement), policy $\pi_1$ will be established as less risky than $\pi_2$ according to the variance criterion. Nevertheless, we know that the cumulative reward earned using $\pi_1$ over $T$ timesteps is bounded from below by $2T$ with probability 1. On the other hand, since $\pi_2$ contains the noise term $h_t$, there is always a positive probability that the reward hits any arbitrarily low value. That is, $\pi_2$ leaves the agent with a component that could constitute a true risk for him, whereas $\pi_1$ doesn't, but is still established as the most risky according to the variance criterion. Hence, the risky noise component is not detected: section \ref{Doob} will establish that this risky component is in fact the martingale associated to the Doob decomposition of the cumulative reward process. On the other hand, the conceptual tool introduced in this paper, the \textit{chaotic variation}, when applied to the mean-variance case, will lead to a risk penalty term of $0$ for $\pi_1$ and proportional to $\sigma^2$ for $\pi_2$ (cf. supplementary), i.e. exactly the desired outcome that the penalty term is proportional to reward uncertainty only.

Another observation is the following: if $\sigma=0$, rewards are deterministic and the obvious optimal policy is to simply take the best reward in each state, i.e. $a_t=2$ if $s_t=1$ and $a_t=1$ if $s_t=2$. This optimal strategy generates a variance over $T$ timesteps of $9T$, due to the switching between states. If one chooses the risk aversion coefficient in the mean-variance trade-off to be high enough, one will then choose not to execute that obvious optimal policy and rather execute policy $\pi_2$ as it generates a lower variance of $4T$ over $T$ timesteps. On the other hand, the \textit{chaotic variation}, when applied to the mean-variance case, will lead to a risk penalty term of $0$ for all policies (since rewards are deterministic) and hence one will always choose the optimal policy, whatever the risk aversion coefficient is.

\begin{figure}[ht]
\begin{center} 
\centerline{\includegraphics[scale=0.4]{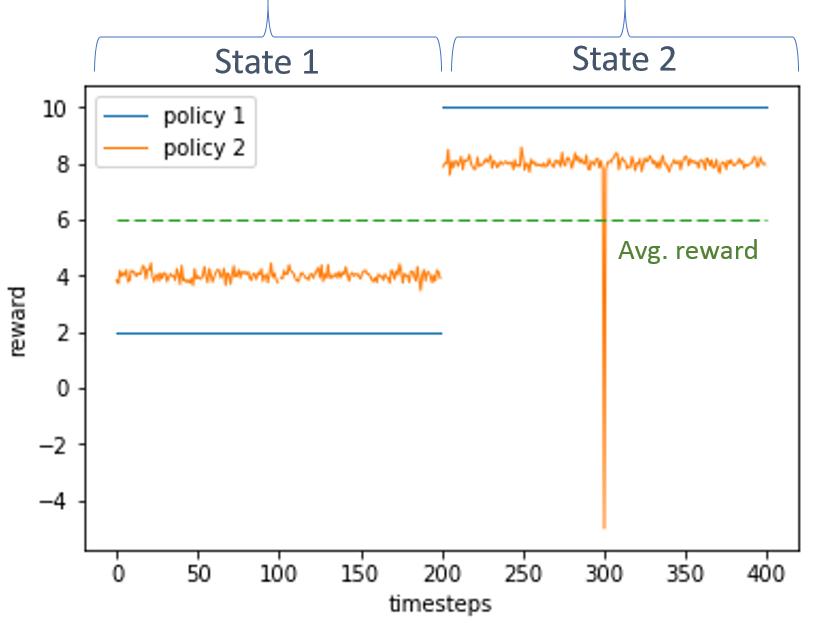}}
\caption{Rewards for 2 policies in state 1 (timesteps 1 to 200) and state 2 (timesteps 200 to 400). State transitions are random but for readability we have grouped together on the plot the timesteps corresponding to each state. The noise $\sigma=0.16$}
\label{fig-toy}
\end{center}
\end{figure}

\subsection{Related work}
\label{relwork}
There exists an interesting and growing body of literature on the topic of risk-sensitive RL. Our work contrasts with the latter in that we focus on learning policies sensitive to the uncertain nature of the rewards whereas previous work apply risk criterion to the cumulative reward as a whole, and hence do not distinguish between the different sources of randomness contained in it. Previously studied algorithms can be split into those that (i) are value-function based (cf. value iteration and policy iteration), and (ii) employ policy gradient methods that update the policy parameters in the direction of the gradient of the risk-flavored performance metric. In category (i), we cite the work of \cite{Mihatsch2002-ck} in which they transform TD increments with a function that is linear by parts, and \cite{Borkar2002}, \cite{Borkar2005}, \cite{Borkar2010-to} which focuses on the exponential utility function in an ergodic average-reward type setting. Work of category (ii) are more numerous, and can be split into those that use (Actor-Critic) and do not use (Monte Carlo) a specific representation of the value function, the former usually requiring TD-type updates of the value function. For large action/state spaces, one typically employs function approximation of the value function in the actor-critic setting, which introduces bias but improves variance. On the RL front, \cite{Prashanth2018-ki} provides a thorough literature review. Some references are interested in the full distribution of the cumulative reward, cf. \cite{Morimura2010-ec}, \cite{Morimura2010-wk}, \cite{Bellemare2017lk}, \cite{Defourny2008-tk} while others focus on specific risk measures, e.g. CVaR \cite{Tamar_A_Glassner_Y_Manor_S2015-zn}, \cite{Chow2018-eq}. Some authors define their own notion of risk with specific applications in mind, such as \cite{Geibel2005-lz} for which risks refers to the likelihood of reaching so-called "error states". \cite{Tamar2015-sj} provides a general methodology to handle all coherent risk measures.

The most-studied framework remains the variance-related one. Variance can be understood as the variance of the cumulative reward, or as the sum of variances of step-by-step rewards, cf. the work of \cite{Prashanth2016-fz}, \cite{JMLR:v17:14-335}, \cite{Tamar2012-ki} (in a stochastic shortest path context), \cite{La2013-pq}. \cite{Sobel_MJ1982-vo} was among the firsts to study the mean-variance case, and he explains why using variance as a measure of risk may cause problems, more specifically the inability to use policy iteration type algorithms due to the lack of monotonicity of the variance operator, while \cite{Mannor2011-vl} shows that these problems can be hard to solve. \cite{JMLR:v17:14-335} provides a methodology to estimate the variance of the cumulative reward via a TD-type approach.

\section{Markov Decision Processes: Doob decomposition and Martingale formulation}
\label{Doob}

We study a Markov Decision Process (MDP) $(\mathbb{S},\mathbb{A},P,R,\gamma,H,\mathbb{H})$, where we allow the reward function $R$ to be general of the form $R_{t+1}:=R(s_t,a_t,s_{t+1},h_{t+1})$, and the hidden process $h_t \in \mathbb{H}$ as in \cite{Shen2013-ds}. The kernels $P$ and $H$ are defined in assumption \ref{a1}. We use the notation that the reward $R_{t+1}$ is received at time $t+1$ after having taken action $a_t \in \mathbb{A}$ at state $s_t \in \mathbb{S}$ at time $t$. We denote the cumulative reward to go and the conditional average reward:
$$
\RR_{t}:=\sum_{t'=t}^\infty \gamma^{t'-t} R_{t'+1}, \hspace{3mm} \overline{R}(s_t,a_t):=\EE[R_{t+1}|s_t,a_t]
$$
To keep notations consistent in the episodic and discounted reward setting we impose as usual $\gamma<1$, and $\gamma=1$ only possible in the episodic setting (in the latter case the process stays in the same absorbing state forever once it is reached, with zero associated rewards). The analysis of this section is amenable to the average-reward formulation, which is detailed in the supplementary. We make the below assumption throughout the paper, needed in particular in the proof of theorem \ref{decomp}.
\begin{assumption}
\label{a1}
$R(\cdot,\cdot,\cdot,\cdot)$ is uniformly bounded and the probability kernels $P$ and $H$ satisfy $\PP[s_{t+1} \in B |s_{t},a_{t},h_t]=P(s_t,a_t,B)$, $\PP[h_{t+1} \in B |s_{t},a_{t},h_t,s_{t+1}]=H(s_t,a_t,s_{t+1},B)$.
\end{assumption}
We need the following notations:
$$
\RR_{n,t}:=\sum_{t'=t}^{(n-1) \vee t} \gamma^{t'-t} R_{t'+1}, \, \FF_n:=\sigma(s_t,a_t,h_t, t \leq n)
$$
$\FF_n$ is the sigma-algebra generated by the MDP before or equal to time $n$, i.e. the information available up to time $n$. 
\begin{definition}
We remind that a discrete-time process $X$ is \textit{adapted} to the filtration $\mathbb{F}:=(\FF_n)_{n \geq 0}$ if $X_n$ is $\FF_n$-measurable for all $n$; it is \textit{predictable} if $X_n$ is $\FF_{n-1}$-measurable for all $n$ (i.e. it is known one timestep before); it is a $\mathbb{F}$-martingale if for all $n$, $X_n$ is adapted to $\FF_n$ and $\EE[X_{n+1}-X_n|\FF_n]=0$.
\end{definition}

The below observation is the cornerstone of this paper, as it will formalize precisely the discussion carried on in section \ref{intro} about the decomposition of the reward variability into two conceptually meaningful components. The martingale component captures the uncertain/stochastic part of the reward in section \ref{toy}. The key idea here is to apply the Doob decomposition of a stochastic process (cf. supplementary) to the reward process in order to gain better understanding of its variability.

\begin{theorem}
\label{decomp}
(Doob decomposition of the reward) Let $\pi$ be a policy. $\RR_{n,t}$ can be decomposed in a unique way $\PP_\pi-a.s.$ into the sum of i) a $\FF_t-$measurable random variable ii) a zero-mean predictable process $\RR_{n,t}^{\pi,pred}$ and iii) a zero-mean martingale $\RR_{n,t}^{chaos}$, with respect to the filtration  $\FF_n$. The decomposition is given by:
$$
\RR_{n,t}=\EE_\pi[\RR_{n,t}|s_t]+\RR_{n,t}^{\pi,pred}+\RR_{n,t}^{chaos}
$$
$$
\RR_{n,t}^{\pi,pred}:=\sum_{t'=t}^{(n-1) \vee t} \gamma^{t'-t} (\overline{R}(s_{t'},a_{t'})-\EE_\pi[R_{t'+1}|s_t])
$$
$$
\RR_{n,t}^{chaos}:=\sum_{t'=t}^{(n-1) \vee t} \gamma^{t'-t} (R_{t'+1}-\overline{R}(s_{t'},a_{t'}))
$$
Denoting $\RR_{t}^{\pi,pred}:=\RR_{\infty,t}^{\pi,pred}$ and $\RR_{t}^{chaos}:=\RR_{\infty,t}^{chaos}$ we get, taking the limit as $n \to +\infty$:
$$
\RR_{t}=\EE_\pi[\RR_{t}|s_t]+\RR_{t}^{\pi,pred}+\RR_{t}^{chaos}
$$
\end{theorem}
\begin{definition}
We call $\RR_t^{chaos}$ (resp. $\RR_t^{\pi,pred}$) the \textit{chaotic} (resp. \textit{predictable}) reward process associated to $\RR_t$.
\end{definition}

The Doob decomposition of theorem \ref{decomp} consists of decomposing rewards into:
\begin{itemize}
    \item the predictable component $\overline{R}(s_{t'},a_{t'})-\EE_\pi[R_{t'+1}|s_t]$ that accounts for deviations of the predictable projection of the uncertain rewards (i.e. the best guess of the unknown reward $R_{t'+1}$ given information at time $t'$) from the overall average reward. This process possesses some variance due to the switching between states, but not to the reward uncertainty.
    \item the chaotic component $R_{t'+1}-\overline{R}(s_{t'},a_{t'})$ exactly captures the uncertainty of the reward by computing the difference between its actual realization and what we expected it to be based on the information available at time $t'$. In other words, it captures the "surprise" part of the reward. $\RR_t^{chaos}=0$ if and only if $R$ is a deterministic reward, in which case we have $R_{t+1}=\overline{R}(s_{t},a_{t})$ with probability 1.
\end{itemize}

\section{Chaotic Variation of the reward process}
\label{chaotic}

We define below a new conceptual tool, the \textit{chaotic variation} associated to a reward process (and to a conditional risk measure $\rho$), that answers the issues raised in section \ref{toy} and captures the reward uncertainty risk. The concept of risk that could be intuited from section \ref{toy} emerges rigorously as a martingale. We refer to \cite{Detlefsen2005-fq} (cf. supplementary) for the definition of a conditional risk-measure associated to a sigma-algebra $\mathcal{G}$, and we use the notation $\rho^\pi_{s_t}$ to denote the conditional risk measure associated to the sigma-algebra $\sigma(s_t)$.

\begin{definition}
\label{defcv}
(Chaotic variation) Let $\pi$ a policy and $\bm{\rho^\pi}:=(\rho^\pi_{s_t})_{t \geq 0}$ a family of conditional risk measures associated to the process $(s_t)_{t \geq 0}$. The chaotic variation associated to ($\RR_t$, $\bm{\rho^\pi}$) is defined as $\CC_{\bm{\rho^\pi}}[\RR_t](s_t):=\rho^\pi_{s_t}(\RR_t^{chaos})$. That is, the chaotic variation quantifies the risk related to the chaotic reward process.
\end{definition}

In the case of the entropic risk measure, the predictable quadratic variation of the chaotic reward process emerges naturally as a measure of risk as a consequence of martingale theory, as proposition \ref{cve} shows it.
\begin{proposition}
\label{cve}
(chaotic variation in the entropic case).
Let $\pi$ a policy. Using definition \ref{defcv}, for $\beta>0$, let $\rho^{\beta,\pi}_{s_t}(X):=\beta^{-1} \ln \EE_\pi [e^{-\beta X}|s_t]$ be the so-called (conditional) entropic risk measure, obtained as the certainty equivalent $CE_{U,\pi,s_t}(X):=U^{-1}(\EE_\pi[U(X)|s_t])$ of the exponential utility function $U(x):=-e^{-\beta x}$. Then, the chaotic variation satisfies $\CC_{\bm{\rho^{\beta,\pi}}}[\RR_t](s_t) \leq \beta^{-1} \ln \sqrt{\EE_\pi[e^{2\beta^2 \left<\RR_t^{chaos}\right>}|s_t]}$,
where $\left<\RR_t^{chaos}\right>$ is the predictable quadratic variation of the martingale $\RR_t^{chaos}$, given by:
$$\left< \RR_t^{chaos}\right>=\sum_{t'=t}^\infty \gamma^{2(t'-t)} \EE[(R_{t'+1}-\overline{R}(s_{t'},a_{t'}))^2|s_{t'},a_{t'}]
$$
As $\beta \to 0$ we get:
$$
\CC_{\rho^{\beta,\pi}}[\RR_t](s_t)=\frac{\beta}{2} \EE_\pi[\left<\RR_t^{chaos}\right>|s_t]+o(\beta)
$$
\end{proposition}

\begin{definition}
\label{mvdef} (Chaotic variance) The chaotic  variance associated to ($\RR_t$, $\pi$) is defined as $
V^{\mathbb{V}(\beta)}_{\pi}(s_t):=\frac{\beta}{2} \EE_\pi[\left<\RR_t^{chaos}\right>|s_t]
$. By martingale property of $\RR_t^{chaos}$, the chaotic variance is equal to the variance of $\RR_t^{chaos}$ (scaled by $\frac{\beta}{2}$, conditional on $s_t$).
\end{definition}

\section{Reinforcement Learning: Risk-Sensitive Chaotic algorithms}
\label{rlchaos}
In this section we assume for simplicity that the state and action spaces are finite. Given the conceptual tools introduced in section \ref{chaotic}, and given a fixed initial state $s_0$ (generalization to the case of distributions over initial states is straightforward), we are interested in solving so-called chaotic problems of type:
\begin{equation}
\label{eqobj}
\max_{\pi} \EE_\pi[\RR_0|s_0]-\CC_{\bm{\rho^\pi}}[\RR_0](s_0)
\end{equation}

The extension to solving problems of type "$\max_{\pi} \EE_\pi[\RR_0|s_0]$ subject to: $\CC_{\bm{\rho^\pi}}[\RR_0](s_0) \leq \lambda$" is not discussed here but can be done using similar techniques developed in \cite{Prashanth2018-ki}, \cite{Tamar2012-ki} or \cite{Prashanth2016-fz}.

\subsection{Bellman equations: Chaotic Mean-Variance case}
\label{secbel}
We present below the Bellman equation in the chaotic mean-variance case $\CC_{\bm{\rho^\pi}}[\RR_0](s_0)=V^{\mathbb{V}(\beta)}_{\pi}(s_0)$, cf. definition \ref{mvdef}. It will be used to formulate a chaotic mean-variance version of Q-Learning in the episodic case (section \ref{secql}), as well as to study actor-critic algorithms and the average reward case (cf. supplementary). The proof of theorem \ref{bel} follows directly from the definition of $\left<\RR_{t}^{chaos}\right>$ in proposition \ref{cve}. 

\begin{theorem}
\label{bel}
(Bellman equation) Let $Q^{\mathbb{V}(\beta)}_{\pi}(s_t,a_t)$ the Q-function associated to $V^{\mathbb{V}(\beta)}_{\pi}(s_t)$ of definition \ref{mvdef}. Then we have the following Bellman equation:
$$
Q^{\mathbb{V}(\beta)}_{\pi}(s_t,a_t)=\EE [ \frac{\beta}{2}(R_{t+1}-\overline{R}(s_{t},a_{t}))^2
+\gamma^2 V^{\mathbb{V}(\beta)}_{\pi}(s_{t+1})|s_t,a_t]
$$
Consequently, if $R^\beta_{t+1}:=R_{t+1}-\frac{\beta}{2} (R_{t+1}-\overline{R}(s_{t},a_{t}))^2$ and $Q_{\pi}(s_t,a_t):=\EE_\pi[\RR_{t}|s_t,a_t]$ with associated value function $V_{\pi}$, then $Q^\beta_{\pi}(s_t,a_t):=Q_{\pi}(s_t,a_t)-Q^{\mathbb{V}(\beta)}_{\pi}(s_t,a_t)$ and its associated value function $V^\beta_\pi$ satisfy the following Bellman equation in the episodic case with $\gamma=1$:
$$
Q^\beta_{\pi}(s_t,a_t)=\EE [ R^\beta_{t+1}
+ V^\beta_{\pi}(s_{t+1})|s_t,a_t]
$$
\end{theorem}

In particular, in the episodic case with $\gamma=1$, theorem \ref{bel} yields that the class of optimal policies associated to (\ref{eqobj}) coincides with that associated to the modified rewards $R^\beta$. In contrast to traditional TD-based methods such as $Q$-Learning, $R^\beta$ is not directly observable as it involves the term $\overline{R}(s_{t},a_{t})$ which will need to be estimated over the course of the algorithm.

\subsection{Chaotic Mean-Variance Q-Learning}
\label{secql}
In the episodic mean-variance case (with $\gamma=1$), theorem \ref{bel} allows us to derive a chaotic version of Q-learning, cf. algorithm \ref{cmvq} which is based on theorem \ref{cmvql}. In the discounted reward setting, it is not possible to combine $Q_{\pi}$ and $Q^{\mathbb{V}(\beta)}_{\pi}$ (where $Q_{\pi}(s_t,a_t):=\EE_\pi[\RR_{t}|s_t,a_t]$), hence we cannot get a Q-Learning type algorithm. The convergence proof is discussed in the supplementary and consists of a minor modification of the proof of \cite{Dayan1992-qg} based on the fact that $\overline{R}_t(s,a) \to \overline{R}(s,a)$ as $t \to +\infty$ with probability 1 for every state-action pair $(s,a)$, where $\overline{R}_t(s,a)$ is defined in theorem \ref{cmvql}. The average reward version of the algorithm is presented in the supplementary.

\begin{theorem}
\label{cmvql}
\textit{(Chaotic Mean-Variance Q-Learning in the episodic case)}.
Using theorem \ref{bel}, denote $Q_{\pi}(s_t,a_t):=\EE_\pi[\RR_{t}|s_t,a_t]$ and $Q^\beta_{\pi}(s_t,a_t):=Q_{\pi}(s_t,a_t)-Q^{\mathbb{V}(\beta)}_{\pi}(s_t,a_t)$. Let $(s_t)$, $(a_t)$ and $(R_{t+1})$ the successive states, actions and rewards observed by the agent. Let $(\alpha_t)$ a sequence of learning rates satisfying the usual conditions for every state-action pair $(s,a)$:
$$
\sum_{k=1}^\infty \alpha_{n_k(s,a)}=+\infty, \hspace{4mm}
\sum_{k=1}^\infty \alpha^2_{n_k(s,a)}<+\infty
$$
where $n_k(s,a)$ is the index corresponding to the $k^{th}$ visit to $(s,a)$. Further define the following iterates if $s_t=s$ and $a_t=a$ :
$$
N_t(s,a)=N_{t-1}(s,a)+1
$$
$$
\overline{R}_t(s,a)=\overline{R}_{t-1}(s,a)+\frac{1}{N_t(s,a)}(R_{t+1}-\overline{R}_{t-1}(s,a))
$$
$$
Q^\beta_{t}(s,a)=(1-\alpha_t)Q^\beta_{t-1}(s,a)
+\alpha_t(R_{t+1}-\frac{1}{2}\beta(R_{t+1}-\overline{R}_t(s,a))^2
+\max_{a'} Q^\beta_{t-1}(s_{t+1},a'))
$$
and $Q^\beta_{t}(s,a)=Q^\beta_{t-1}(s,a)$, $N_{t}(s,a)=N_{t-1}(s,a)$, $\overline{R}_{t}(s,a)=\overline{R}_{t-1}(s,a)$ otherwise. Then $Q^\beta_t(s,a) \to Q^{\beta}_*(s,a)$ as $t \to +\infty$ with probability 1 for every state-action pair $(s,a)$, where $ Q^{\beta}_*(s,a):=\sup_\pi Q^\beta_{\pi}(s,a)$.
\end{theorem}

\begin{remark}
\label{rqsa}
Note that in theorem \ref{cmvql}, one could use a two-timescale stochastic approximation algorithm by using a specific timescale for the $\overline{R}_t$ process (instead of $\frac{1}{N_t(s,a)}$). Since the update rule of $\overline{R}_t$ doesn't depend on $Q^\beta_{t}(s,a)$ (uncoupled case), this is not required for convergence but could be used to improve the rate of convergence, cf. a remark in \cite{aapkt}, p. 4.
\end{remark}

\begin{algorithm}[ht]
\caption{\textbf{CMV-Q-Learning (episodic case)}}\label{cmvq}
\textbf{Input:} {$Q^\beta$-table initialized arbitrarily, learning rate $(\alpha_t)_{t \geq 0}$, $\overline{R}(s,a)$ and $N(s,a)$ initialized to 0.}\\
\textbf{Output:} {optimal policy $\pi_*(s)=\mbox{argmax}_a Q^\beta_*(s,a)$}
\begin{algorithmic}[1]
    \For{each episode}
    \State{initialize $s_t=s_0$}
    \While{$s_t$ is not terminal}
	    \State{Choose $a_t$ from $s_t$ using a policy derived from $Q^\beta$ (e.g. $\epsilon$-greedy).}
	    \State{Take action $a_t$, observe $s_{t+1}$, $R_{t+1}$.}
	    \State{$N(s_t,a_t) \gets N(s_t,a_t)+1$; $\overline{R}(s_t,a_t) \gets  \overline{R}(s_t,a_t)+\frac{1}{N(s_t,a_t)}(R_{t+1}-\overline{R}(s_t,a_t))$}
        \State{$Q^\beta(s_t,a_t) \gets (1-\alpha_t) Q^\beta(s_{t},a_t) + \alpha_t( R_{t+1}-\frac{\beta}{2}(R_{t+1}-\overline{R}(s_t,a_t))^2+\max_a Q^\beta(s_{t+1},a))$}
        \State{$s_t \gets s_{t+1}$}
        \EndWhile
        \EndFor
\end{algorithmic}
\end{algorithm}

\subsection{Monte Carlo Policy gradient algorithms - Episodic case}

In this section we consider episodic Monte Carlo based algorithms that start with a parametric form $\pi_\theta$ for the policy and optimize equation (\ref{eqobj}) in the direction of the gradient with respect to $\theta$, hence aiming for local optima only. We make the following classical assumption in this section (\cite{Bhatnagar2009-ie}):  \begin{assumption}
\label{a2}
For every policy $\pi_\theta$, the Markov chain induced by $P$ and $\pi_\theta$ is ergodic, i.e. irreducible, aperiodic and positive recurrent. Further, for every state-action pair $(s,a)$, $\pi_\theta(s,a)$ is continuously differentiable in $\theta$.
\end{assumption}
From equation (\ref{eqobj}), we need to compute unbiased estimates of $\nabla_\theta \EE_{\pi_\theta}[\RR_0|s_0]$ and $\nabla_\theta\CC_{\bm{\rho^{{\pi_\theta}}}}[\RR_0](s_0)$. The former is the classical expected reward gradient. By definition \ref{defcv}, $\CC_{\bm{\rho^{{\pi_\theta}}}}[\RR_0](s_0)$ consists in applying a risk measure to a mean zero martingale, which is usually simpler as we will see below in the mean-variance case $\CC_{\bm{\rho^{{\pi_\theta}}}}[\RR_0](s_0)=V^{\mathbb{V}(\beta)}_{\pi_\theta}(s_0)$ (cf. definition \ref{mvdef}), and more importantly the work done in the literature for specific risk measures can be applied straightforwardly, e.g. \cite{Chow2018-eq} or \cite{Tamar_A_Glassner_Y_Manor_S2015-zn} in the case of CVaR (however the risk measure in our case is applied to the chaotic reward process $\RR_t^{chaos}$ only). From now on we focus on the mean-variance case of definition \ref{mvdef}, but we present in the supplementary extensions to chaotic CVaR and Sharpe ratio. 

Provided we know $\overline{R}(s,a)$, the gradient $\nabla_\theta V^{\mathbb{V}(\beta)}_{\pi_\theta}(s_0)$ presents no specific difficulty and can be computed using the classical likelihood ratio technique by generating $B \geq 1$ episodes $s_0^{(b)}$, $a_0^{(b)}$, $R_1^{(b)}$, ..., $s_{T_b-1}^{(b)}$, $a_{T_b-1}^{(b)}$, $R_{T_b}^{(b)}$, $b=1..B$, following $\pi_\theta$ and computing the unbiased Monte Carlo estimate:
$$
\nabla_\theta V_{\pi_\theta}^{\mathbb{V}(\beta)}(s_0)=\frac{\beta}{2}\frac{1}{B} \sum_{b=1}^B \sum_{t'=0}^{T_b-1} \nabla \ln \pi_\theta(a^{(b)}_{t'}|s^{(b)}_{t'}) V_{b,t'}
$$
$$
V_{b,t'}:=\sum_{t=t'}^{T_b-1} \gamma^{2(t-t')}(R^{(b)}_{t+1}-\overline{R}(s^{(b)}_t,a^{(b)}_t))^2
$$

The subtlety here is that we need to learn $\overline{R}(s,a)$ in the course of the algorithm. In the tabular case, using assumption \ref{a2}, we are guaranteed that every state-action pair $(s,a)$ will be visited infinitely often and we can proceed as in the Q-Learning case (theorem \ref{cmvql}) by updating the (conditional) average estimate $\widehat{R}(s,a)$ as iterations on $\theta$ are performed. By the strong law of large numbers, convergence of $\widehat{R}(s,a)$ to $\overline{R}(s,a)$ occurs with probability 1. Remark \ref{rqsa} still holds here: since the iteration on the sample average $\widehat{R}(s,a)$ is uncoupled from the one on $\theta$, we do not need a two-timescale algorithm to guarantee convergence. For each $b=1..B$ and $t=0..T_b$, if $s=s_t^{(b)}$ and $a=a_t^{(b)}$ we perform the updates:
\begin{equation}
\begin{aligned}
\label{eqe1}
&N(s,a)=N(s,a)+1\\
&\widehat{R}(s,a)=\widehat{R}(s,a)+\frac{1}{N(s,a)}(R^{(b)}_{t+1}-\widehat{R}(s,a))
\end{aligned}
\end{equation}
Alternatively, if the state or action spaces are large, one may want to use a parametric approximation $\widehat{R}_\psi(s,a)$ of $\overline{R}(s,a)$ that will be updated in the course of the algorithm: this approach is presented in algorithm \ref{eqe2}, where we use an experience replay table and fit $\widehat{R}_\psi$ using SGD. Proposition \ref{bias} quantifies the related gradient bias due to the use of this approximation (proof in supplementary). We present in algorithm \ref{cmvr} the Chaotic Mean-Variance version of REINFORCE, which uses either equation (\ref{eqe1}) or algorithm \ref{eqe2}. 

\begin{proposition}
\label{bias}
Let $\epsilon_\psi(s,a):=\widehat{R}_\psi(s,a)-\overline{R}(s,a)$ and $\nabla_\theta V^{\mathbb{V}(\beta)}_{\pi_\theta,\psi}$ the gradient obtained using the approximation $\widehat{R}_\psi$. The gradient bias $B_{\psi}^\theta(s_0)$$:=\nabla_\theta V^{\mathbb{V}(\beta)}_{\pi_\theta,\psi}(s_0)-\nabla_\theta V^{\mathbb{V}(\beta)}_{\pi_\theta}(s_0)$ satisfies:
$$
B_{\psi}^\theta(s_0)=\frac{\beta}{2}\EE_{(S,A)\sim d^\theta_{\gamma^2}(s_0)}[\nabla \ln \pi_\theta(A|S) b^\psi_{\theta}(S,A)]
$$
$$
b^\psi_{\theta}(s,a):=\EE_{\pi_\theta}[\sum_{t=0}^\infty \gamma^{2t} \epsilon_\psi^2(s_t,a_t)|s_0=s,a_0=a]
$$
where $d^\theta_{\gamma^2}(s_0,s,a):=\pi_\theta(a|s)\sum_{t=0}^{\infty}\gamma^{2t} \mathbb{P}_{\pi_\theta}[s_t=s|s_0]$ is the action-state $\gamma^2$-discounted visiting distribution.
\end{proposition}

\begin{algorithm}[ht]
\caption{\textbf{Distributional Update}}\label{eqe2}
\textbf{Input:} {initial distributional parameter $\psi$, experience replay table $\EEE$, number of distributional gradient steps $N_\psi$, number of SGD samples  $M_\psi$.\\
\textbf{Output:} {Approximation of the optimal distributional parameter $\psi^*$}
\begin{algorithmic}[1]
\State{Draw $M_\psi$ samples $(s_j,a_j,R_j)$ randomly from $\EEE$.}
\State{Set $\widetilde{M}_\psi$ to be the number of unique pairs $(\widetilde{s}_j,\widetilde{a}_j)$ and for each such pair, set $\widetilde{R}_j$ to be the average of the corresponding $R_j$.}
\State{Perform $N_\psi$ steps of SGD on the loss ($\widehat{R}_\psi(\widetilde{s}_j,\widetilde{a}_j)-\widetilde{R}_j)^2$ using the $\widetilde{M}_\psi$ samples.}
\end{algorithmic}
}
\end{algorithm}

\begin{algorithm}[ht]
\caption{\textbf{CMV-REINFORCE (episodic case)}}\label{cmvr}
\textbf{Input:} {Initial policy parameter $\theta_0$, learning rate $(\alpha_n)$, number of episodes per batch $B$, $\widehat{R}_\psi$ initialized to 0, Optional: experience replay table $\EEE$.}\\
\textbf{Output:} {Approximation of the optimal policy parameter $\theta^*$}
\begin{algorithmic}[1]
	\While{$\theta_n$ not converged}
	    \State{Generate $B$ episodes $s_0^{(b)}$, $a_0^{(b)}$, $R_1^{(b)}$, ..., $s_{T_b-1}^{(b)}$,
	    $a_{T_b-1}^{(b)}$, $R_{T_b}^{(b)}$, $b=1..B$, following $\pi_{\theta_n}(\cdot|\cdot)$}
	    \State{\textbf{In tabular case}: use episodes to update $\widehat{R}_\psi$ with eq. (\ref{eqe1}); \textbf{else} increment $\EEE$ with the tuples ($s_{t}^{(b)}$,
	    $a_{t}^{(b)}$, $R_{t}^{(b)}$) and update $\widehat{R}_\psi$ using algorithm \ref{eqe2}.}
	    \For{$b = 1$ to $B$}
	    \State{$v_{t,b}:=\sum_{t'=t}^{T_b-1} \gamma^{t'-t} R^{(b)}_{t'+1}$$-\frac{\beta}{2}\gamma^{2(t'-t)}(R^{(b)}_{t'+1} - \widehat{R}_\psi(s^{(b)}_{t'},a^{(b)}_{t'}))^2$}
        \State{$V_b \gets \sum_{t=0}^{T_b-1} \nabla \ln \pi_{\theta_n}(a^{(b)}_{t}|s^{(b)}_{t})v_{t,b}$}
	    \EndFor
        \State{$\theta_{n+1} \gets \theta_n+\alpha_n B^{-1}\sum_{b=1}^B V_b$}
    \EndWhile
\end{algorithmic}
\end{algorithm}

\subsection{Policy Gradient Actor-Critic algorithms}

In order to derive Actor-Critic algorithms, the Bellman equations in theorem \ref{bel} can be used - as in the CMV-Q-Learning algorithm \ref{cmvq} - in order to adapt the work of e.g. \cite{Prashanth2016-fz} (variance case) or \cite{Chow2018-eq} (CVaR case) in order to derive chaotic versions of their actor-critic algorithms. This extension is discussed in the supplementary.

\section{Experiments: Chaotic Mean-Variance}
\label{appl}
\subsection{Grid World}
\label{apgs}
We consider the episodic problem of a robot on a grid aiming at a goal. The state space consists of the 16 grid squares, and the action space consists in choosing to go East, West, North or South. Reaching the goal (resp. taking a step) gives a +1 (resp. -1) reward and negative rewards are positioned on the grid as seen on figure \ref{fig-robot1}. When an action $a_t$ is chosen, there is a probability $p_{error}=50\%$ that the robot goes in a random direction. If the robot hits the extremity of the grid, it stays where it is. We train policies using the CMV-Q-Learning algorithm \ref{cmvq}. In figure \ref{fig-robot1} we plot the path heatmap over $10^5$ rollout steps performed with the learned policy for various risk aversion coefficients $\beta$ (cf. details and additional experiments in supplementary). The higher $\beta$, the further away from the -20 reward the robot goes, as expected, as it prefers taking the -6 loss rather than walking next to the -20 reward and risking to encur the corresponding loss. 

The second element that our algorithm gives is a \textit{risk heatmap, highlighting the most uncertain states}, i.e. the states which yield the highest reward stochasticity/uncertainty (figure \ref{fig-robot1}). We perform rollouts with the learned $\overline{R}(\cdot,\cdot)$ and policy $\pi^*_\beta$ and compute $\EE_{\pi^*_\beta}[(R_{t+1}-\overline{R}(s,\pi^*_\beta(s)))^2|s_t=s]$ for every state $s$. In figure \ref{fig-robot1} we see that as expected, the heatmap highlights the states next to the -20 reward, and to a lesser intensity the states around the -6 rewards. To the best of our knowledge, only two other work study risk-sensitive Q-Learning algorithms: \cite{Mihatsch2002-ck} and \cite{Shen2013-ds}. Both transform the TD-increments by a non-linear function in order to obtain risk-sensitive behaviors. The former show that their algorithm converge to the worst-case optimality criterion as the risk parameter changes, hence in our specific example, we could obtain similar paths as in figure \ref{fig-robot1}, but our algorithm, in addition to being the only one to focus on reward stochasticity only, additionally provides the risk heatmap discussed above, which quantifies the extent to which a given state yields uncertain rewards.

\begin{figure}[ht]
\vskip 0.1in
\begin{center} 
\centerline{\includegraphics[scale=0.5]{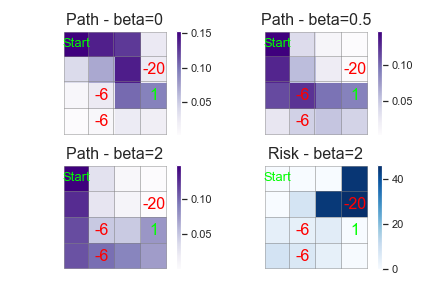}}
\caption{Path and risk heatmaps over $10^5$ rollout steps with the learned policy, $p_{error}=50\%$ (cf. description in text).}
\label{fig-robot1}
\end{center}
\vskip -0.1in
\end{figure}

\subsection{Portfolio optimization}
\label{apstock}
We consider the problem of investing in 2 financial assets. One is risk-free in the sense that investing a quantity $q^{RF}_{t}$ in the asset yields a deterministic reward $R^{RF}_{t+1}:=q^{RF}_{t}\mu(s_t)$, where $\mu$ is the \textit{risk-free rate}. The latter can be seen as the overnight rate from day $t$ to day $t+1$, which is known at the end of day $t$ when $q^{RF}_{t}$ is chosen. The other asset is risky (e.g. a stock) and yields an uncertain reward for a quantity $q^{R}_{t}$ of $R^{R}_{t+1}:=q^{R}_{t}(\mu(s_t)+\sigma(s_t)h_{t+1})$, where $(h_{t})$ are i.i.d. standard normal and $\sigma$ is referred to as \textit{volatility}. We restrict $q^{RF}_{t},q^{R}_{t}$ to be nonnegative integers which sum is less than a total investment constraint $q_{max}$. The action is defined as $a_t:=(q^{RF}_{t},q^{R}_{t})$ and the reward $R_{t+1}:=R^{RF}_{t+1}+R^{R}_{t+1}$. In our experiments we take $q_{max}=5$, yielding $21$ possible actions. The state space consists of 3 states \textit{LowVol}, \textit{HighVol} and \textit{MediumVol} defined as low volatility $\sigma$ (and low risk-free rate $\mu$), high volatility (and high risk-free rate $\mu$), as well as an intermediate state. The state transition matrix is designed such that the more we trade in the risky asset (i.e. the higher $q^{R}_{t}$), the more likely we are to reach a higher volatility state. We train the policy using CMV-REINFORCE algorithm \ref{cmvr} and its baseline (classical mean-variance), and use it to compute performance metrics over rollout episodes of $T$ timesteps. Details on training/numerical values used are presented in the supplementary.

\begin{figure}[ht]
\vskip 0.1in
\begin{center} 
\centerline{\includegraphics[scale=0.4]{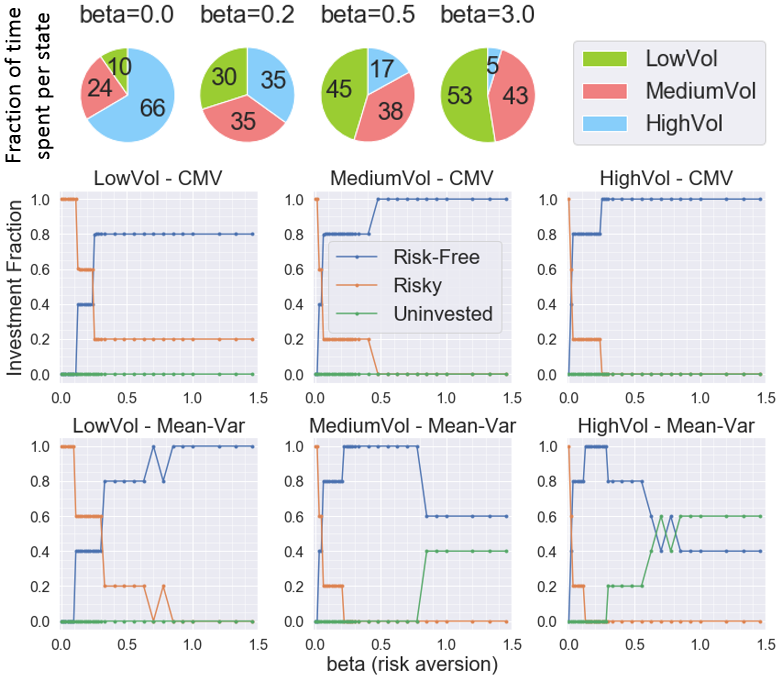}}
\caption{\textbf{(Top)} Fraction of time spent per state \textbf{(Mid/Bottom)} CMV and baseline (mean-variance) asset investment fraction per state as a function of $\beta$ - $T=20$ timesteps, $s_0=\mbox{LowVol}$.}
\label{fig-cmv}
\end{center}
\vskip -0.1in
\end{figure}

We display in figure \ref{fig-cmv} the fraction of time the process has spent in each state, as well as the investment fraction in each asset (for rollout episodes). With $\beta=0$ and by design of the state transition matrix, both assets give the same expected reward $\mu$ but the policy is incentivized to trade in the risky asset rather than in the risk-free asset in order to reach the HighVol state, which gives the highest $\mu$. As $\beta$ increases, in the CMV case, the policy shifts towards trajectories which contain less reward uncertainty, which means investing in the risk-free asset which associated rewards are deterministic. The classical mean-variance case (baseline) penalizes not only the variability related to reward uncertainty but also the variability related to the switching of states appearing in both assets via the deterministic term $\mu(s_t)$, hence as $\beta$ increases, the policy is incentivized to stop investing (cf. green line in figure \ref{fig-cmv} representing the uninvested budget $q_{max}-q^{RF}_{t}-q^{R}_{t}$), i.e. it stops taking advantage of the risk-free asset. The latter can also be seen in figure \ref{fig-r20} where we plot for rollout episodes the cumulative reward mean and standard deviation as a function of $\beta$. In the mean-variance case the reward mean will eventually vanish as the policy will stop trading in both assets: by trying to make the standard deviation lower, it generates a counterintuitive behavior in that it stops taking advantage of the risk-free asset.

\section{Conclusion}
We presented a novel, conceptually meaningful decomposition of the cumulative reward process based on the Doob decomposition that distinguishes between the different sources of randomness contained within it, introduced a new conceptual tool - the \textit{chaotic variation} - that exactly captures reward uncertainty risk, and incorporated it into model-free value-function based and policy gradient algorithms. Potential real-world applications include all settings where one is subject to uncertain/stochastic rewards and is interested in deriving interpretable risk-sensitive policies, for example recently studied RL financial market-making problems \cite{GM}, \cite{nvmm} where reward uncertainty plays a major role in that market-makers stream prices but do not know whether clients will decide to trade at that price, and further they are typically averse to uncertain fluctuations in the underlying financial asset price. Future work could include extending the framework to the case of delayed rewards.

\begin{figure}[t]
\vskip 0.1in
\begin{center} 
\centerline{\includegraphics[scale=0.4]{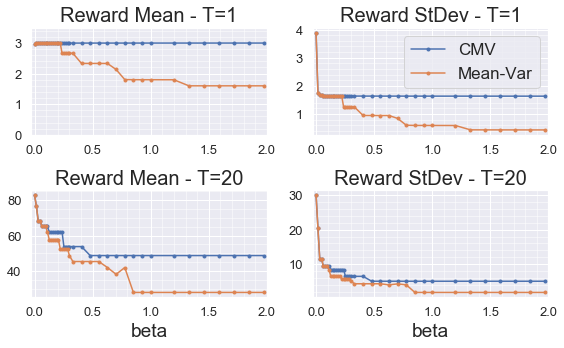}}
\caption{Cumulative reward Mean and Std.Dev. for CMV and baseline (mean-variance), as a function of $\beta$ - \textbf{(Top)} $T=1$, $s_0=\mbox{Random}$ and \textbf{(Bottom)} $T=20$ timesteps, $s_0=\mbox{LowVol}$. }
\label{fig-r20}
\end{center}
\vskip -0.1in
\end{figure}

\section*{Disclaimer}
\normalsize

This paper was prepared for information purposes by the Artificial Intelligence Research group of JPMorgan Chase \& Co and its affiliates (“JP Morgan”), and is not a product of the Research Department of JP Morgan. JP Morgan makes no representation and warranty whatsoever and disclaims all liability, for the completeness, accuracy or reliability of the information contained herein. This document is not intended as investment research or investment advice, or a recommendation, offer or solicitation for the purchase or sale of any security, financial instrument, financial product or service, or to be used in any way for evaluating the merits of participating in any transaction, and shall not constitute a solicitation under any jurisdiction or to any person, if such solicitation under such jurisdiction or to such person would be unlawful. \textsuperscript{\textcopyright} 2020 JPMorgan Chase \& Co. All rights reserved.

\small

\bibliography{neurips_2020}
\bibliographystyle{authordate1}

\normalsize

\appendix

\section{Background}
\label{secbac}
\subsection{Central limit theorem for Markov chain functionals (section 1 - Introduction)}
Assume for simplicity that the state space $\mathbb{S}$ is finite, and let $R_\pi(s_t,s_{t+1}):=R(s_t,\pi(s_t),s_{t+1})$ the reward obtained at time $t+1$ associated to the deterministic policy $\pi$, so that $R_\pi$ is a function of $s_t$ and $s_{t+1}$ only, and $(s_t)$ is a Markov chain satisfying $\PP[s_{t+1}=s'|s_t=s]=P(s,\pi(s),s')=:P_\pi(s,s')$. The central limit theorem for Markov chain functionals (\cite{nnv}, theorem 3.17) states that the limit in distribution as $n \to +\infty$ of:
$$
\sqrt{n} \left(\frac{1}{n}\sum_{t=0}^n R_\pi(s_t,s_{t+1})-\sum_{x,y \in \mathbb{S}}d_{\pi}(x)P_\pi(x,y)R_\pi(x,y)\right)
$$
is normal with mean zero and variance $\sigma_{deter.}^2+\sigma_{chaotic}^2$, where $d_\pi$ is the stationary distribution of the Markov chain of transition kernel $P_\pi$. Equivalently, the latter limit is equal in distribution to the sum of two independent normal random variables of respective variances $\sigma_{deter.}^2$ and $\sigma_{chaotic}^2$, which are given by the following expressions:
$$
\sigma_{chaotic}^2=\sum_{x \in \mathbb{S}}d_{\pi}(x)P_\pi^{var}R_{\pi}(x,\cdot)(x), \hspace{5mm}
\sigma_{deter.}^2=\sum_{x \in \mathbb{S}}d_{\pi}(x)P_\pi^{var}f_{Poisson}(x)
$$
Denoting the conditional expected reward $\mu_\pi(x):=\EE[R_{\pi}(s_t,s_{t+1})|s_t=x]=\sum_{y \in \mathbb{S}} R_{\pi}(x,y)P_\pi(x,y)$, we see below that $\sigma_{deter.}^2$ only depends on the rewards via the deterministic term $\mu_\pi(x)$, whereas $\sigma_{chaotic}^2$ additionally depends on the term $\EE[R_{\pi}(s_t,s_{t+1})^2|s_t=x]$ which quantifies reward stochasticity/uncertainty. Indeed, the variance operators $P_\pi^{var}$ act as follows:
$$
P_\pi^{var}R_{\pi}(x,\cdot)(x)=\sum_{y \in \mathbb{S}} R^2_{\pi}(x,y)P_\pi(x,y)
-\mu_\pi(x)^2
$$
and:
$$
P_\pi^{var}f_{Poisson}(x)=\sum_{y \in \mathbb{S}} f^2_{Poisson}(y)P_\pi(x,y)
-(\sum_{y \in \mathbb{S}} f_{Poisson}(y)P_\pi(x,y))^2
$$
where the function $f_{Poisson}$ is the solution of the Poisson equation:
$$
\sum_{y \in \mathbb{S}} f_{Poisson}(y)P_\pi(x,y)-f_{Poisson}(x)
=\sum_{x'\in \mathbb{S}}d_{\pi}(x')\mu_\pi(x')-\mu_\pi(x)
$$

\subsection{Doob decomposition of a stochastic process}
\label{bdoob}
Let $X$ a discrete-time process on a probability space $(\Omega,\FF,\PP)$ such that (i) $\EE[|X_n|]<\infty$ for all $n$, and (ii) $X$ is adapted to the filtration $\mathbb{F}:=(\FF_n)_{n \geq 0}$, i.e. $X_n$ is $\FF_n$-measurable for all $n$. Then, there exists an integrable martingale $M$, and an integrable predictable process $A$, such that $A_0=M_0=0$ and $X_n=X_0+A_n+M_n$ for all $n$. This decomposition is almost surely unique. Here, we remind that $A$ being predictable means that $A_n$ is $\FF_{n-1}$-measurable for all $n$ (i.e. it is known one timestep before), and $M$ martingale means that $M_n$ is $\FF_n$-measurable and $\EE[M_{n+1}-M_n|\FF_n]=0$. Further, $A$ and $M$ are given by:
$$
A_n=\sum_{k=1}^n \EE[X_k|\FF_{k-1}]-X_{k-1}, \hspace{5mm}
M_n=\sum_{k=1}^n X_k-\EE[X_k|\FF_{k-1}]
$$

\subsection{Conditional risk measures}

We remind here the definition of a conditional risk measure associated to a sigma-algebra $\mathcal{G}$, according to \cite{Detlefsen2005-fq}. Let $L^\infty$, $L^\infty_{\mathcal{G}}$ the set of resp. bounded random variables and bounded, $\mathcal{G}$-measurable random variables. A conditional risk measure $\rho$ associated to a sigma-algebra $\mathcal{G}$ is a map $L^\infty \to L^\infty_{\mathcal{G}}$ such that:
\begin{itemize}
    \item (Normalization) $\rho(0)=0$.
    \item (Conditional Translation Invariance) For any $X \in L^\infty$, $Z \in L^\infty_{\mathcal{G}}$ we have $\rho(X+Z)=\rho(X)-Z$.
    \item (Monotonicity) For any $X,Y \in L^\infty$, if $X \leq Y$ with probability 1, then $\rho(X)\geq \rho(Y)$ with probability 1.
\end{itemize}

\section{Proofs}
\label{secb}
\subsection{Proof of theorem 1}
Define the process $Y_n=\RR_{n+t,t}-\EE_\pi[\RR_{n+t,t}|s_t]$ for $n \geq 1$ and $Y_0:=0$, where we remind that:
$$
\RR_{n,t}:=\sum_{t'=t}^{(n-1) \vee t} \gamma^{t'-t} R_{t'+1}
$$
The process $Y_n$ is adapted to the filtration generated by the sigma-algebras $\GG_n:=\FF_{n+t}$, since by definition $\FF_n:=\sigma(s_k,a_k,h_k, k \leq n)$. Hence, we can apply the Doob decomposition (cf. section \ref{bdoob}), and we get the following decomposition of the process $Y_{n}$ up to unicity $\mathbb{P}_\pi-$a.s., for $n \geq 0$:
$$
Y_{n}=Y_{0}+Y_{n}^{pred}+Y_{n}^{chaos}
$$
where $Y_{n}^{pred}$, $Y_{n}^{chaos}$ are respectively a predictable process and a martingale with respect to the filtration $(\GG_n)_{n \geq 0}$, satisfying $Y_{0}^{pred}=Y_{0}^{chaos}=0$, and are given by, for $n \geq 0$ (with the usual convention that $\sum_{k=1}^{0}(\cdot):=0$):
$$Y_{n}^{pred}=\sum_{k=1}^{n}(\EE_\pi[Y_{k}|\GG_{k-1}]-Y_{k-1}), \hspace{5mm}Y_{n}^{chaos}=\sum_{k=1}^{n}(Y_{k}-\EE_\pi[Y_{k}|\GG_{k-1}])$$ 
We have $Y_{0}=0$ and if $k \geq 2$:
$$
\EE_\pi[Y_{k}|\GG_{k-1}]-Y_{k-1}=\EE_\pi[\RR_{k+t,t}-\RR_{k+t-1,t}|\FF_{k+t-1}]
-\EE_\pi[\RR_{k+t,t}-\RR_{k+t-1,t}|s_t]$$$$
=\gamma^{k-1} (\EE_\pi[R_{k+t}|\FF_{k+t-1}] - \EE_\pi[R_{k+t}|s_t])
$$
If $k=1$ we get:
$$\EE_\pi[Y_{k}|\GG_{k-1}]-Y_{k-1}=\EE_\pi[Y_{1}|\GG_{0}]=\EE_\pi[R_{t+1}|\FF_{t}] - \EE_\pi[R_{t+1}|s_t]$$ 
Overall, this gives for $n \geq 1$:
$$
Y_{n}^{pred}=\sum_{k=1}^{n}\gamma^{k-1} (\EE_\pi[R_{k+t}|\FF_{k+t-1}] - \EE_\pi[R_{k+t}|s_t])
=\sum_{k=t}^{n-1+t}\gamma^{k-t} (\EE_\pi[R_{k+1}|\FF_{k}] - \EE_\pi[R_{k+1}|s_t])
$$
Now, we claim that $\EE_\pi[R_{k+1}|\FF_{k}]=\overline{R}(s_k,a_k)$. Indeed, by assumption 1 and since by definition $\FF_n:=\sigma(s_k,a_k,h_k, k \leq n)$, we get:
$$
\EE_\pi[R_{k+1}|\FF_{k}]=\EE_\pi[R_{k+1}|s_m,a_m,h_m : m \leq k]$$$$
=\int_{\mathbb{S}} \int_{\mathbb{H}} R(s_k,a_k,s',h') H(s_k,a_k,s',dh') P(s_k,a_k,ds')
=\EE[R_{k+1}|s_k,a_k]=\overline{R}(s_k,a_k)
$$
This yields for $n \geq 1$:
$$
Y_{n}^{pred}=\sum_{k=t}^{n-1+t}\gamma^{k-t} (\overline{R}(s_k,a_k) - \EE_\pi[R_{k+1}|s_t])=\RR_{n+t,t}^{\pi,pred}
$$
Similarly for $n \geq 1$: 
$$
Y_{n}^{chaos}=\sum_{k=t}^{n-1+t}\gamma^{k-t} (R_{k+1}-\overline{R}(s_k,a_k))=\RR_{n+t,t}^{chaos}
$$
Since $Y_{n}^{pred}$, $Y_{n}^{chaos}$ are respectively a predictable process and a martingale with respect to the filtration $(\GG_n)_{n \geq 0}$ by the Doob decomposition, and by definition $\GG_n:=\FF_{n+t}$, we get that $\RR_{n,t}^{\pi,pred}$, $\RR_{n,t}^{chaos}$ are respectively a predictable process and a martingale with respect to the filtration $\FF_n$, for $n \geq t+1$, and satisfy:
$$
\RR_{n,t}-\EE_\pi[\RR_{n,t}|s_t]=Y_{n-t}=Y_{n-t}^{pred}+Y_{n-t}^{chaos}=\RR_{n,t}^{\pi,pred}+\RR_{n,t}^{chaos}
$$

\subsection{Proof of proposition 1}

We have the Taylor expansion:
$$
e^{-\beta \RR_t^{chaos}}=1-\beta \RR_t^{chaos}+\frac{\beta^2}{2}(\RR_t^{chaos})^2+o(\beta^2)
$$
We remind that the quadratic variation $\left< M_n\right>$ of a discrete-time, square-integrable martingale $M$ adapted to a filtration $(\FF_n)_{n \geq 0}$ satisfies $\EE[\left< M_n\right>]=\EE[M_n^2]$ and is given by:
$$\left< M_n\right>=\sum_{k=1}^n  \EE[(M_k-M_{k-1})^2|\FF_{k-1}]
$$

By theorem 1, the process $\RR_{n,t}^{chaos}$ is a mean zero martingale, hence $\EE_\pi[\RR_{t}^{chaos}|s_t]=0$ by the dominated convergence theorem and the process $Z_{n}:=(\RR_{n,t}^{chaos})^2-\left<\RR_{n,t}^{chaos}\right>$ is a mean zero martingale, where $\left<\RR_{n,t}^{chaos}\right>$ is the predictable quadratic variation of the martingale $\RR_{n,t}^{chaos}$, given by:
$$\left< \RR_{n,t}^{chaos}\right>=\sum_{t'=t}^{(n-1) \vee t} \gamma^{2(t'-t)} \EE[(R_{t'+1}-\overline{R}(s_{t'},a_{t'}))^2|s_{t'},a_{t'}]
$$
This yields $\EE_\pi[Z_{\infty}|s_t]=\EE_\pi[Z_0|s_t]=0$, that is:
$$
\EE_\pi[(\RR_t^{chaos})^2|s_t]=\EE_\pi[\left<\RR_t^{chaos}\right>|s_t]
$$
All terms put together we get:
$$
\rho^{\beta,\pi}_{s_t}(\RR_t^{chaos})=\beta^{-1} \ln\left(1+\frac{\beta^2}{2}\EE_\pi[\left<\RR_t^{chaos}\right>|s_t]+o(\beta^2)\right)
=\frac{\beta}{2}\EE_\pi[\left<\RR_t^{chaos}\right>|s_t]+o(\beta)
$$
We now proceed to proving that:
$$\CC_{\bm{\rho^{\beta,\pi}}}[\RR_t](s_t) \leq \beta^{-1} \ln \sqrt{\EE_\pi[e^{2\beta^2 \left<\RR_t^{chaos}\right>}|s_t]}$$
Since the logarithm function is increasing, it is sufficient to prove that $\EE_\pi[e^{-\beta \RR_t^{chaos}}|s_t] \leq  \sqrt{\EE_\pi[e^{2\beta^2 \left<\RR_t^{chaos}\right>}|s_t]}$. Since $-\beta \RR_{n,t}^{chaos}$ is a (bounded) martingale, we get that $Y_{n}:=\exp(-2\beta\RR_{n,t}^{chaos}-2\beta^2\left<\RR_{n,t}^{chaos}\right>)$ is a martingale, namely the exponential martingale associated to $-2\beta \RR_{n,t}^{chaos}$. We then have:
$$
e^{-\beta \RR_{n,t}^{chaos}}=e^{-\beta \RR_{n,t}^{chaos}-\beta^2\left<\RR_{n,t}^{chaos}\right>}e^{\beta^2\left<\RR_{n,t}^{chaos}\right>}
$$
By H\"older's inequality and taking the limit as $n \to \infty$ (using the dominated convergence theorem), we get:
$$
\EE_\pi[e^{-\beta \RR_{t}^{chaos}}|s_t] \leq \sqrt{\EE_\pi[\underbrace{e^{-2\beta \RR_{t}^{chaos}-2\beta^2\left<\RR_{t}^{chaos}\right>}}_{Y_\infty}|s_t]}\sqrt{\EE_\pi[e^{2\beta^2\left<\RR_{t}^{chaos}\right>}|s_t]}
$$
Since $Y_n$ is a martingale, $\EE_\pi[Y_\infty|s_t]=\EE_\pi[Y_0|s_t]=1$, which gives the result.

\subsection{Proof of convergence of Chaotic Mean-Variance Q-Learning algorithm of theorem 3.}
By theorem 2 we have:
$$
Q^{\mathbb{V}(\beta)}_{\pi}(s_t,a_t)=\EE [\frac{\beta}{2} (R_{t+1}-\overline{R}(s_{t},a_{t}))^2
+V^{\mathbb{V}(\beta)}_{\pi}(s_{t+1})|s_t,a_t]
$$
We also have the classical Bellman equation:
$$
Q_{\pi}(s_t,a_t)=\EE [R_{t+1}+V_{\pi}(s_{t+1})|s_t,a_t]
$$
Hence $Q^\beta_{\pi}(s_t,a_t):=Q_{\pi}(s_t,a_t)-Q^{\mathbb{V}(\beta)}_{\pi}(s_t,a_t)$ (and its associated value function $V^\beta_\pi$) satisfies the classical Bellman equation with modified rewards $R^\beta$:
$$
Q^{\beta}_{\pi}(s_t,a_t)=\EE [ R^\beta_{t+1}
+V^{\beta}_{\pi}(s_{t+1})|s_t,a_t]
$$
where:
$$
R^\beta_{t+1}:=R_{t+1}-\frac{\beta}{2} (R_{t+1}-\overline{R}(s_{t},a_{t}))^2
$$
By assumption of the theorem, $\sum_{k=1}^\infty \alpha_{n_k(s,a)}=+\infty$, hence the $k^{th}$ visit index to $(s,a)$ $n_k(s,a) \to +\infty$ as $k \to +\infty$ and so every state-action pair is visited infinitely often. As a consequence we get that $N_t(s,a) \to +\infty$ as $t \to +\infty$ with probability 1 and by the strong law of large numbers, $\overline{R}_t(s,a) \to \overline{R}(s,a)$ as $t \to +\infty$ with probability 1. This guaranties in the proof of \cite{Dayan1992-qg} that in step B.3 of Lemma B (\textit{Rewards and transition probabilities converge with probability 1}), using their notations, the expected rewards $\mathcal{R}_s^{(t)}(a)$ of the so-called action-replay process (ARP) tend as $t \to +\infty$ to the expected rewards $\EE[R^\beta_{t+1}|s_t=s,a_t=a]$ of the real process. The rest of the convergence proof of \cite{Dayan1992-qg} goes through the same way. Note that in the latter reference, the proof is given in the discounted reward case with $\gamma<1$, but their section 4 "Discussions and Conclusions" discusses the proof extension to the episodic case with $\gamma=1$.

\subsection{Proof of proposition 2}
\label{proofbias}
We have by definition:
$$
Q^{\mathbb{V}(\beta)}_{\pi_\theta}(s,a)=\frac{\beta}{2}\EE_{\pi_\theta}[\left<\RR_{0}^{chaos}\right>|s_0=s,a_0=a]
$$
and $V^{\mathbb{V}(\beta)}_{\pi_\theta}$ the associated value function, as in definition 4. We have the below equality, which proof is very similar to that \cite{Prashanth2016-fz} (lemma 1). We postpone it at the end of the present proof for reader's convenience:
$$
\nabla_\theta V^{\mathbb{V}(\beta)}_{\pi_\theta}(s_0)=
\EE_{(S,A)\sim d^\theta_{\gamma^2}(s_0)}[\nabla \ln \pi_\theta(A|S) Q^{\mathbb{V}(\beta)}_{\pi_\theta}(S,A)]
$$
where $d^\theta_{\gamma^2}(s_0,s,a):=\pi_\theta(a|s) \sum_{t=0}^{\infty}\gamma^{2t} \mathbb{P}_{\pi_\theta}[s_t=s|s_0]$ is the action-state $\gamma^2$-discounted visiting distribution. Similarly we have, with $\overline{R}$ replaced by $\widehat{R}$:
$$
\nabla_\theta V^{\mathbb{V}(\beta)}_{\pi_\theta,\psi}(s_0)=
\EE_{(S,A)\sim d^\theta_{\gamma^2}(s_0)}[\nabla \ln \pi_\theta(A|S) Q^{\mathbb{V}(\beta)}_{\pi_\theta,\psi}(S,A)]
$$
where:
$$
Q^{\mathbb{V}(\beta)}_{\pi_\theta,\psi}(s,a):=\frac{\beta}{2}\EE_{\pi_\theta}[\left<\RR_{0,\psi}^{chaos}\right>|s_0=s,a_0=a]
$$
$$\left< \RR_{0,\psi}^{chaos}\right>=\sum_{t=0}^\infty \gamma^{2t} \EE[(R_{t+1}-\widehat{R}_\psi(s_{t},a_{t}))^2|s_{t},a_{t}]
$$
By definition we have:
$$
\left<\RR_{0,\psi}^{chaos}\right>-\left<\RR_{0}^{chaos}\right>
=\sum_{t=0}^\infty \gamma^{2t} (\widehat{R}_\psi(s_{t},a_{t})-\overline{R}(s_{t},a_{t})) \EE[\widehat{R}_\psi(s_t,a_t)+\overline{R}(s_t,a_t)-2R_{t+1}|s_{t},a_{t}]
$$
But:
$$
\EE[\widehat{R}_\psi(s_t,a_t)+\overline{R}(s_t,a_t)-2R_{t+1}|s_{t},a_{t}]
=\widehat{R}_\psi(s_t,a_t)+\overline{R}(s_t,a_t)-2\overline{R}(s_t,a_t)
=\widehat{R}_\psi(s_t,a_t)-\overline{R}(s_t,a_t)
$$
Hence:
$$
\left<\RR_{0,\psi}^{chaos}\right>-\left<\RR_{0}^{chaos}\right>
=\sum_{t=0}^\infty \gamma^{2t} (\widehat{R}_\psi(s_{t},a_{t})-\overline{R}(s_{t},a_{t}))^2
$$
which shows that:
$$
B_{\psi}^\theta(s_0)=\frac{\beta}{2}\EE_{(S,A)\sim d^\theta_{\gamma^2}(s_0)}[\nabla \ln \pi_\theta(A|S) b^\psi_{\pi_\theta}(S,A)]
$$
with:
$$
b^\psi_{\pi_\theta}(s,a):=\EE_{\pi_\theta}[\sum_{t=0}^\infty \gamma^{2t} (\widehat{R}_\psi(s_{t},a_{t})-\overline{R}(s_{t},a_{t}))^2|s_0=s,a_0=a]
$$

Finally, we prove as claimed earlier that:
$$
\nabla_\theta V^{\mathbb{V}(\beta)}_{\pi_\theta}(s_0)=
\EE_{(S,A)\sim d^\theta_{\gamma^2}(s_0)}[\nabla \ln \pi_\theta(A|S) Q^{\mathbb{V}(\beta)}_{\pi_\theta}(S,A)]
$$
The proof is very similar to that of \cite{Prashanth2016-fz} (lemma 1). We first use Bellman equation for $Q^{\mathbb{V}(\beta)}_{\pi_\theta}$ (theorem 2):
$$
Q^{\mathbb{V}(\beta)}_{\pi_\theta}(s_t,a_t)=\EE [\frac{\beta}{2}(R_{t+1}-\overline{R}(s_{t},a_{t}))^2
+\gamma^2 V^{\mathbb{V}(\beta)}_{\pi_\theta}(s_{t+1})|s_t,a_t]
$$
hence taking the gradient:
$$
\nabla_\theta Q^{\mathbb{V}(\beta)}_{\pi_\theta}(s_t,a_t)=\gamma^2 \EE [\nabla_\theta V^{\mathbb{V}(\beta)}_{\pi_\theta}(s_{t+1})|s_t,a_t]
=\gamma^2 \int_{\mathbb{S}}\nabla_\theta V^{\mathbb{V}(\beta)}_{\pi_\theta}(s')P(s_t,a_t,s')ds'
$$

On the other hand by definition of the value function:
$$
\nabla_\theta V^{\mathbb{V}(\beta)}_{\pi_\theta}(s_t)=\nabla_\theta \int_{\mathbb{A}} \pi_\theta(a|s_t)Q^{\mathbb{V}(\beta)}_{\pi_\theta}(s_t,a)da
=\int_{\mathbb{A}} \nabla_\theta \pi_\theta(a|s_t)Q^{\mathbb{V}(\beta)}_{\pi_\theta}(s_t,a)da
+\int_{\mathbb{A}} \pi_\theta(a|s_t)\nabla_\theta Q^{\mathbb{V}(\beta)}_{\pi_\theta}(s_t,a)da
$$
Plugging in the expression obtained for $\nabla_\theta Q^{\mathbb{V}(\beta)}_{\pi_\theta}(s_t,a_t)$ we get:
$$
\nabla_\theta V^{\mathbb{V}(\beta)}_{\pi_\theta}(s_t)=\int_{\mathbb{A}} \left[ \nabla_\theta \pi_\theta(a|s_t)Q^{\mathbb{V}(\beta)}_{\pi_\theta}(s_t,a) \right.
\left. +\gamma^2\pi_\theta(a|s_t)\int_{\mathbb{S}}\nabla_\theta V^{\mathbb{V}(\beta)}_{\pi_\theta}(s')P(s_t,a,s')ds' \right]da
$$
After unrolling $\nabla_\theta V^{\mathbb{V}(\beta)}_{\pi_\theta}(s')$ infinitely many times we get:
$$
\nabla_\theta V^{\mathbb{V}(\beta)}_{\pi_\theta}(s_t)=
\int_{\mathbb{A}} \int_{\mathbb{S}}\sum_{t'=t}^{\infty} \gamma^{2(t'-t)}\mathbb{P}_{\theta}[s_{t'}=s|s_t] \nabla_\theta \pi_\theta(a|s)Q^{\mathbb{V}(\beta)}_{\pi_\theta}(s,a) ds da
$$
Since $\nabla_\theta \pi_\theta(a|s)=\pi_\theta(a|s) \nabla_\theta \ln\pi_\theta(a|s)$, and using the definition of $d^\theta_{\gamma^2}(s_t,s,a)$, we get that $\nabla_\theta V^{\mathbb{V}(\beta)}_{\pi_\theta}(s_t)$ is equal to:
$$
\int_{\mathbb{A}} \int_{\mathbb{S}}\nabla_\theta \ln \pi_\theta(a|s)Q^{\mathbb{V}(\beta)}_{\pi_\theta}(s,a) d^\theta_{\gamma^2}(s_t,s,a) ds da
$$
which completes the proof.

\section{Toy example of section 1.2: generalization and quantitative results}
\label{secc}
We slightly generalize the example in section 1.2 by considering the case where there are $N$ states, 2 actions, and at each state transition the probability to reach another state is given by $P(s_{t+1}=n|s_t,a_t)=P(s_{0}=n)=p_n \in [0,1]$ such that $\sum_{n=1}^N p_n=1$ and the reward are as follows:
\begin{itemize}
\item if action 1 is chosen, the reward received at $t+1$ if $s_t=n$ is the constant $\mu_n \in \mathbb{R}$.
\item if action 2 is chosen, the reward received at $t+1$ if $s_t=n$ is $\mu_n+\kappa_n+\sigma_n h_{t+1}$, where $(h_{t})$ are zero mean and unit variance i.i.d. and $\kappa_n,\sigma_n \in \mathbb{R}$.
\end{itemize}
Hence the reward has the compact formulation:
$$
R(s_{t}=n,a_t,s_{t+1},h_{t+1})=\mu_n+(\kappa_n+\sigma_n h_{t+1})1_{\{a_t=2\}}
$$
where we denote the indicator function by $1_{\{\cdot\}}$. The below example \ref{counter} formulates quantitatively the informal discussion in section 1.2 by showing that, using cumulative reward variance as a measure of risk and provided the noise is small enough, a policy that leaves the agent with a truly risky component $h_t$ (i.e. that can hit any arbitrarily low value with positive probability) may be established as less risky than a policy that doesn't, hence that "true risk" fails to be captured using this standard criterion.
\begin{example}
\label{counter}
(Risk fails to be captured using cumulative reward variance as a measure of risk). Let us consider the regime-switching example described above, denote $\VV$ the variance operator and introduce the classical variance penalty $v_\pi(\RR_0)$ (variance of the sum is equal to the sum of variances in this specific example as the rewards $R_t$ are independent):
$$
v_\pi(\RR_0):=\VV_\pi\left[\sum_{t=1}^T R_t \right]=\sum_{t=1}^T \VV_\pi[R_t]
$$
If $\sum_{n=1}^N p_n \kappa_n=0$ and $\pi_i$ is the policy that always selects action $i$, we get:
$$
v_{\pi_2}(\RR_0)-v_{\pi_1}(\RR_0)=T\sum_{n=1}^N p_n \left( \kappa_n^2+2\mu_n\kappa_n+\sigma_n^2\right)
$$
In particular, $v_{\pi_2}(\RR_0)-v_{\pi_1}(\RR_0)$ can be negative. For example if $N=2$, $p_n=0.5$, $\mu_2=\mu_1+\delta$, $\kappa_2=-\delta \epsilon=-\kappa_1$, then:
$
v_{\pi_2}(\RR_0)-v_{\pi_1}(\RR_0)=T\delta^2 \left( \epsilon^2-\epsilon+\frac{1}{2\delta^2}(\sigma_1^2+\sigma_2^2)\right)
$
i.e. the latter is negative provided the average noise is small enough $\frac{1}{2}(\sigma_1^2+\sigma_2^2)<\frac{1}{4}\delta^2$.
\end{example}
\begin{proof}
We have, with $\pi(n):=\mathbb{P}[a_t=2|s_t=n]$:
$$
\sum_{t=1}^T \VV_\pi[R_t]=\sum_{t=1}^T\sum_{n=1}^N p_n \EE_\pi[R_t^2|s_{t-1}=n]
-\sum_{t=1}^T(\sum_{n=1}^N p_n \EE_\pi[R_t|s_{t-1}=n])^2
$$
$$
=T\sum_{n=1}^N p_n
(\mu_n^2+(\kappa_n^2+\sigma_n^2+2\mu_n \kappa_n)\pi(n))
-T(\sum_{n=1}^N p_n (\mu_n+ \pi(n) \kappa_n))^2
$$
By definition of $\pi_1$ and $\pi_2$, we have $\pi_1(n)=0$ and $\pi_2(n)=1$ for all $n$ and hence:
$$
v_{\pi_2}(\RR_0)-v_{\pi_1}(\RR_0)
=T\sum_{n=1}^N p_n
(\kappa_n^2+\sigma_n^2+2\mu_n \kappa_n)
$$
$$
-(\sum_{n=1}^N p_n (\mu_n+\kappa_n))^2+(\sum_{n=1}^N p_n\mu_n)^2
$$
Since by assumption $\sum_{n=1}^N p_n \kappa_n=0$, we get the desired result. If $N=2$, $p_n=0.5$, $\mu_2=\mu_1+\delta$, $\kappa_2=-\delta \epsilon=-\kappa_1$ then:
$$
v_{\pi_2}(\RR_0)-v_{\pi_1}(\RR_0)
$$$$=\frac{T}{2}(2\delta^2 \epsilon^2+\sigma_1^2+\sigma_2^2+2\mu_1\delta \epsilon - 2\delta \epsilon(\mu_1+\delta))
=T\delta^2(\epsilon^2+\frac{\sigma_1^2+\sigma_2^2}{2\delta^2}- \epsilon)
$$
The latter is a 2nd order polynomial in $\epsilon$ with positive $\epsilon^2$ coefficient, hence it can take negative values if and only if it has two distinct real roots, i.e. if $\frac{1}{2}(\sigma_1^2+\sigma_2^2)<\frac{1}{4}\delta^2$.
\end{proof}

Proposition \ref{rschao} below shows that the chaotic variance $V^{\mathbb{V}(\beta)}_{\pi}$ is proportional to the hidden noise, and in particular is zero in the absence of such noise. That is, chaotic variance captures the risky component $h_t$ contained in the rewards.

\begin{proposition}
\label{rschao}
The chaotic variance (cf. definition 4) associated to the regime-switching example \ref{counter} is given by:
$$
V^{\mathbb{V}(\beta)}_{\pi}=\frac{\beta}{2}T \sum_{n=1}^N \sigma^2_n p_n \pi(n)
$$
where $\pi(n):=\PP[a_t=2|s_t=n]$ and $V^{\mathbb{V}(\beta)}_{\pi}:=\sum_{n=1}^N V^{\mathbb{V}(\beta)}_{\pi}(n)p_n$.
\end{proposition}
\begin{proof}
By definition:
$$
\left<\RR_0^{chaos}\right>= \sum_{t=1}^T \EE_\pi[(R_t-\EE_\pi[R_t|s_{t-1},a_{t-1}])^2|s_{t-1},a_{t-1}]
$$
By definition of $R_t$ we get:
$$
\EE_\pi[R_t|s_{t-1},a_{t-1}]=\mu_{s_{t-1}}+(\kappa_{s_{t-1}}+\sigma_{s_{t-1}} h_{t})1_{\{a_{t-1}=2\}})
$$
and hence:
$$
R_t-\EE_\pi[R_t|s_{t-1},a_{t-1}]=1_{\{a_{t-1}=2\}}\sigma_{s_{t-1}} h_{t}
$$
so that:
$$
\EE_\pi[(R_t-\EE_\pi[R_t|s_{t-1},a_{t-1}])^2|s_{t-1},a_{t-1}]=1_{\{a_{t-1}=2\}}\sigma^2_{s_{t-1}}
$$
and therefore:
$$
\left<\RR_0^{chaos}\right>= \sum_{t=1}^T 1_{\{a_{t-1}=2\}}\sigma^2_{s_{t-1}}
$$
By definition the chaotic variance is given by $V^{\mathbb{V}(\beta)}_{\pi}(n)=\frac{\beta}{2} \EE_\pi[\left<\RR_0^{chaos}\right>|s_0=n]$, and hence:
$$
V^{\mathbb{V}(\beta)}_{\pi}(n)=\frac{\beta}{2}\EE_\pi[1_{\{a_{0}=2\}}\sigma^2_{s_{0}}|s_0=n]
+\frac{\beta}{2}\sum_{t=2}^T\EE_\pi[1_{\{a_{t-1}=2\}}\sigma^2_{s_{t-1}}|s_0=n]
$$
But $\EE_\pi[1_{\{a_{0}=2\}}\sigma^2_{s_{0}}|s_0=n]=\pi(n)\sigma^2_n$ and for $t \geq 2$:
$$
\EE_\pi[1_{\{a_{t-1}=2\}}\sigma^2_{s_{t-1}}|s_0=n]=\sum_{k=1}^N\pi(k)\sigma^2_k
p_k$$
Plugging in the latter expression we get all in all:
$$
V^{\mathbb{V}(\beta)}_{\pi}(n)=\frac{\beta}{2}\pi(n)\sigma^2_n+\frac{\beta}{2}(T-1)\sum_{k=1}^N\pi(k)\sigma^2_k p_k
$$
Since by definition $V^{\mathbb{V}(\beta)}_{\pi}=\sum_{n=1}^N V^{\mathbb{V}(\beta)}_{\pi}(n)p_n$ we get eventually:
$$
V^{\mathbb{V}(\beta)}_{\pi}=\frac{\beta}{2}\sum_{n=1}^N\pi(n)\sigma^2_n p_n+\frac{\beta}{2}(T-1)\sum_{n=1}^N p_n\sum_{k=1}^N\pi(k)\sigma^2_k p_k
=T\frac{\beta}{2}\sum_{n=1}^N p_n\pi(n)\sigma^2_n
$$
\end{proof}

\section{Average reward version of the Chaotic Mean-Variance Q-Learning algorithm (theorem 3)}
\label{AQL}

In the average reward framework, some modifications are required since $\sum_{t=0}^\infty R_{t+1}$ doesn't necessarily converge anymore. We follow the spirit of \cite{Prashanth2016-fz}, and start by defining 
$$
\rho_\pi:=\lim_{n \to \infty}\frac{1}{n}\sum_{t=0}^n R_{t+1}
,\hspace{3mm}
\widetilde{\RR}_{t}:=\sum_{t'=t}^\infty (R_{t'+1}-\rho_\pi)
$$
$$
Q_{\pi}(s_t,a_t)=\EE_\pi[\widetilde{\RR}_{t}|s_t,a_t]
$$
The chaotic variance process needs to be modified similarly according to a recentering by $\sigma_\pi$: 
$$
\sigma_\pi:=\lim_{n \to \infty}\frac{1}{n}\sum_{t=0}^n (R_{t+1}-\overline{R}(s_{t},a_{t}))^2
$$ 
$$
\RR^{\mathbb{V}(\beta)}_{t}:=\frac{\beta}{2}\sum_{t'=t}^\infty ((R_{t+1}-\overline{R}(s_{t},a_{t}))^2-\sigma_\pi)
$$
$$
Q^{\mathbb{V}(\beta)}_{\pi}(s_t,a_t)=\EE_\pi[\RR^{\mathbb{V}(\beta)}_{t}|s_t,a_t]
$$ 

In that case we get the following Bellman equations, similar to theorem 2, which will be used in theorem \ref{cmvrl}:
$$
Q_{\pi}(s_t,a_t)=\EE [R_{t+1}-\rho_\pi
+\gamma V_{\pi}(s_{t+1})|s_t,a_t]
$$
$$
Q^{\mathbb{V}(\beta)}_{\pi}(s_t,a_t)=\EE [ \frac{\beta}{2}(R_{t+1}-\overline{R}(s_{t},a_{t}))^2-\frac{\beta}{2}\sigma_\pi
+\gamma^2 V^{\mathbb{V}(\beta)}_{\pi}(s_{t+1})|s_t,a_t]
$$
We present below the chaotic mean-variance version of the R-learning algorithm in \cite{aql}. The algorithm is similar to the one presented in the episodic case (theorem 3), except that the rewards and the chaotic variance have to be adjusted by their mean value as precised above (ergodic limit). The mean values $\rho_\pi$ and $\sigma_\pi$ have to be estimated during the course of the algorithm.

\begin{theorem}
\label{cmvrl}
\textit{(Chaotic Mean-Variance R-Learning in the average reward case)}. Let $Q^\beta_{\pi}(s_t,a_t):=Q_{\pi}(s_t,a_t)-Q^{\mathbb{V}(\beta)}_{\pi}(s_t,a_t)$ and $(s_t)$, $(a_t)$ and $(R_{t+1})$ the successive states, actions and rewards observed by the agent. Let $(\alpha^{(1)}_t)$, $(\alpha^{(2)}_t)$ be two sequences of learning rates. The Chaotic Mean-Variance R-Learning algorithm is given by:
\begin{equation*}
\begin{aligned}
&N_t(s,a)=N_{t-1}(s,a)+1\\
&\overline{R}_t(s,a)=\overline{R}_{t-1}(s,a)+\frac{1}{N_t(s,a)}(R_{t+1}-\overline{R}_{t-1}(s,a))\\
&Q^\beta_{t}(s,a)=(1-\alpha^{(1)}_t)Q^\beta_{t-1}(s,a)+\alpha^{(1)}_t(R_{t+1}-\rho_{t-1}
-\frac{1}{2}\beta((R_{t+1}-\overline{R}_t(s,a))^2-\sigma_{t-1})
+\max_{a'} Q^\beta_{t-1}(s_{t+1},a'))\\
&\rho_{t}=(1-\alpha^{(2)}_t)\rho_{t-1}
+\alpha^{(2)}_t(R_{t+1}+\max_{a'}Q^\beta_{t-1}(s_{t+1},a)-\max_{a'}Q^\beta_{t-1}(s,a))\\
&\sigma_{t}=(1-\alpha^{(2)}_t)\sigma_{t-1}+\alpha^{(2)}_t((R_{t+1}-\overline{R}_t(s,a))^2
+\max_{a'}Q^\beta_{t-1}(s_{t+1},a)-\max_{a'}Q^\beta_{t-1}(s,a))
\end{aligned}
\end{equation*}
if $s_t=s$ and $a_t=a$ and $Q^\beta_{t}(s,a)=Q^\beta_{t-1}(s,a)$, $N_{t}(s,a)=N_{t-1}(s,a)$, $\overline{R}_{t}(s,a)=\overline{R}_{t-1}(s,a)$ otherwise.
\end{theorem}

\section{Episodic Monte Carlo Policy gradient algorithms for additional risk measures: the chaotic CVaR and Sharpe Ratio cases}
\label{secee}

Here we discuss the extension of CMV-REINFORCE algorithm 3 to the chaotic Sharpe ratio and CVaR cases. 

\subsection{The chaotic Sharpe Ratio case}
\label{CSR}
The chaotic Sharpe ratio - which we seek to maximize - is defined as:
$$
\CC_{\pi}^{Sh}[\RR_t](s_0):=\frac{V_{\pi}(s_0)}{\sqrt{V^{\mathbb{V}}_{\pi}(s_0)}}
$$
where we remind that $V_{\pi}(s_t):=E_\pi[\RR_t|s_t]$, $V^{\mathbb{V}}_{\pi}(s_t):=V^{\mathbb{V}(2)}_{\pi}(s_t)=\EE_\pi[\left<\RR_t^{chaos}\right>|s_t]$, and we assume that $V^{\mathbb{V}}_{\pi}(s_0)>0$ $\mathbb{P}_\pi$-a.s. Taking the gradient, we get: 
$$
\nabla_\theta \CC_{\pi}^{Sh}[\RR_t](s_0)=\frac{\nabla_\theta V_\pi(s_0)}{\sqrt{V^{\mathbb{V}}_{\pi}(s_0)}}  -\frac{1}{2}\frac{V_\pi(s_0) \nabla_\theta V^{\mathbb{V}}_{\pi}(s_0)}{V^{\mathbb{V}}_{\pi}(s_0)^{3/2}}
$$
Provided $V_{\pi}(s_0)$ and $V^{\mathbb{V}}_{\pi}(s_0)$ are known, we can compute unbiased estimates of the gradients $\nabla_\theta V_\pi(s_0)$, $\nabla_\theta V^{\mathbb{V}}_\pi(s_0)$ as done in algorithm 3, in particular updating $\widehat{R}_\psi$ the same way. In order to estimate $V_{\pi}(s_0)$ and $V^{\mathbb{V}}_{\pi}(s_0)$, we use the techniques developed in \cite{Tamar2012-ki} (theorem 4.3) which uses two timescales: $V_{\pi}(s_0)$ and $V^{\mathbb{V}}_{\pi}(s_0)$ are estimated on the faster timescale $\alpha^{(2)}_n$ so that they can be considered as converged when the $\theta$ update is performed on the slower timescale $\alpha^{(1)}_n$. We impose $\lim_{n \to \infty} \frac{\alpha^{(1)}_n}{\alpha^{(2)}_n}=0$ and we perform the fast timescales updates:
$$
V_{n+1}=V_n+\frac{\alpha^{(2)}_n}{B} \sum_{b=1}^{B} \sum_{t=0}^{T_b-1} \gamma^{t}R^{(b)}_{t+1}
$$
$$
V^{\mathbb{V}}_{n+1}=V^{\mathbb{V}}_n+\frac{\alpha^{(2)}_n}{B} \sum_{b=1}^{B} \sum_{t=0}^{T_b-1} \gamma^{2t} (R^{(b)}_{t+1} - \widehat{R}_\psi(s^{(b)}_t,a^{(b)}_t))^2
$$
The unbiased estimates of the gradients are given as in algorithm 3 by:
$$
\nabla V_n =\frac{1}{B} \sum_{b=1}^{B} \sum_{t'=0}^{T_b-1} \nabla \ln \pi_{\theta_n}(a^{(b)}_{t'}|s^{(b)}_{t'}) \sum_{t=t'}^{T_b-1} \gamma^{t} R^{(b)}_{t+1}
$$
$$
\nabla V^{\mathbb{V}}_n =\frac{1}{B} \sum_{b=1}^{B} \sum_{t'=0}^{T_b-1} \nabla \ln \pi_{\theta_n}(a^{(b)}_{t'}|s^{(b)}_{t'}) 
\sum_{t=t'}^{T_b-1}  \gamma^{2t}(R^{(b)}_{t+1} - \widehat{R}_\psi(s^{(b)}_t,a^{(b)}_t))^2
$$	   
Finally, $\theta$ is updated on the slow timescale by:
$$
\theta_{n+1}=\theta_n+ \alpha^{(1)}_n \left(\frac{\nabla V_n}{\sqrt{V^{\mathbb{V}}_n}}  -\frac{1}{2}\frac{V_n \nabla V^{\mathbb{V}}_n}{V^{\mathbb{V},3/2}_n}\right)
$$
For completeness we present the assumptions under which the above algorithm is guaranteed to converge, cf. theorem 4.3 in \cite{Tamar2012-ki} which proof uses two timescales and an ODE based approach:
\begin{itemize}
    \item $\lim_{n \to \infty} \frac{\alpha^{(1)}_n}{\alpha^{(2)}_n}=0$, and for $j=1,2$: $\sum_{n=0}^\infty \alpha^{(j)}_n=+\infty$, $\sum_{n=0}^\infty \alpha^{(j)2}_n<+\infty$.
    \item Assumptions 1, 2 hold true (which guarantee in particular that $V_\pi$ and $V^{\mathbb{V}}_\pi$ are uniformly bounded).
    \item For all $\theta$, the objective function $f_{\theta}$ has bounded second derivatives. Furthermore, the set of local optima of $f_{\theta}$ is countable. Here the objective function $f_{\theta}$ is defined as:
    $$
    f_{\theta}(s)=\CC_{\pi_\theta}^{Sh}[\RR_t](s)
    $$
\end{itemize}

\subsection{The chaotic CVaR case}
\label{CSR2}

We recall that the $CVaR_\beta$ of a random variable $X$ is defined as $CVaR_\beta(X):=\EE[X|X \leq Var_\beta(X)]$, where the value-at-risk $Var_\beta(X)=F_X^{-1}(X)$, where $F_X$ is the cumulative distribution function of $X$. We adapt the work \cite{Chow2018-eq} (algorithm 1) to the chaotic case, i.e. the case where $X=\mathcal{R}_t^{chaos}$ is the chaotic reward process. In order to obtain a unbiased estimate of the gradient of
$$
\CC_{\pi_\theta}^{CVaR(\beta)}[\RR_t](s_0):=
\EE_{\pi_\theta}[\mathcal{R}_t^{chaos}(s_0)|s_t,\mathcal{R}_t^{chaos}(s_0) \leq Var_\beta(\mathcal{R}_t^{chaos}(s_0))]
$$
we first estimate as in \cite{Chow2018-eq} (algorithm 1) $Var_\beta(\mathcal{R}_t^{chaos}(s_t))$ on a fast timescale $\alpha^{(2)}_n$ so that it can be considered as converged when performing the $\theta$ update:
$$
Var_{n+1}=Var_{n}-\alpha^{(2)}_n(1-(1-\beta)^{-1} \frac{1}{B} \sum_{b=1}^{B} 1\{Z_b \geq Var_{n}\})
$$
$$
Z_b:=\sum_{t=0}^{T_b-1} \gamma^{2t} (R^{(b)}_{t+1} - \widehat{R}_\psi(s^{(b)}_t,a^{(b)}_t))
$$
We can then compute an unbiased estimate of the gradient $\widehat{\nabla CVaR}$ of $\nabla \CC_{\pi_\theta}^{CVaR(\beta)}[\RR_t](s_0)$ as:
$$
\widehat{\nabla CVaR_n}=\frac{1}{B}\sum_{b=1}^{B}\nabla \ln \pi_{\theta_n,b}(Z_b-Var_{n}) 1\{Z_b \geq Var_{n}\}
$$
$$
\nabla \ln \pi_{\theta_n,b}:= \sum_{t=0}^{T_b-1} \nabla \ln \pi_{\theta_n}(a^{(b)}_{t}|s^{(b)}_{t}) 
$$
Eventually, $\theta$ is updated on the slow timescale $\alpha^{(1)}_n$ as:
$$
\theta_{n+1}=\theta_n+\alpha^{(1)}_n(\frac{1}{B}\sum_{b=1}^{B}J_b - \widehat{\nabla CVaR_n})
$$
where $J_b$ is as in algorithm 3.

\section{Policy Gradient Actor-Critic algorithms}
\label{secAC}

In the episodic, discounted or average reward settings, we can use the Bellman equations in theorem 2 to adapt the work of e.g. \cite{Prashanth2016-fz} (variance case) or \cite{Chow2018-eq} (CVaR case) in order to derive chaotic versions of their actor-critic algorithms. 

In the chaotic mean-variance discounted reward case, and under assumption 2, the policy gradient theorem gives us for the expected reward and risk-sensitive gradients:
\begin{equation}
\begin{aligned}
\label{eqac}
&\nabla_\theta \EE_\pi[\RR_0|s_0]=
\EE_{(S,A)\sim d^\theta_{\gamma}(s_0)}[\nabla \ln \pi_\theta(A|S) Q_\pi(S,A)]\\
&\nabla_\theta V^{\mathbb{V}(\beta)}_{\pi_\theta}(s_0)=
\frac{\beta}{2} \EE_{(S,A)\sim d^\theta_{\gamma^2}(s_0)}[\nabla \ln \pi_\theta(A|S) Q^{\mathbb{V}(\beta)}_\pi(S,A)]
\end{aligned}
\end{equation}

where $d^\theta_{\gamma^n}(s_0,s,a):=\sum_{t=0}^{\infty}\gamma^{nt} \pi_\theta(a|s)\mathbb{P}_{\pi_\theta}[s_t=s|s_0]$ is the action-state $\gamma^n$-discounted visiting distribution. The proof of the above is very similar to \cite{Prashanth2016-fz} (lemma 1), and we have derived it in the proof of section \ref{proofbias} for the convenience of the reader. In the average reward case, under assumption 2 and using the functions $Q_\pi$ and chaotic variance $V^{\mathbb{V}(\beta)}_\pi$ defined in section \ref{AQL}, we get the same equalities except that $(S,A) \sim (d^\theta,\pi_\theta)$, where  $d^\theta$ is the stationary distribution of the underlying Markov chain generated by having actions follow $\pi_\theta$. The proof of this result is very similar to that of equations (\ref{eqac}) and is given in e.g. \cite{Prashanth2016-fz}, lemma 3.

Note that in equations (\ref{eqac}), the $Q$ functions can, as usual, be replaced by the corresponding advantage functions by subtracting the value functions.

In order to derive online Actor-Critic algorithms, we follow the approach in \cite{Prashanth2016-fz} by defining $V^{\phi}_\pi$, $V^{\phi,\mathbb{V}(\beta)}_\pi$ as parametric approximations of the value function and chaotic variance, where $\phi$ is a feature vector. We then apply theorem 2 to compute Temporal Differences that are used to estimate the advantage functions in equations (\ref{eqac}), and to update the critic parameters.

\begin{remark}
In the average reward setting, the policy gradient equations (\ref{eqac}) can be used to derive an online Actor-Critic algorithm as in \cite{Prashanth2016-fz}, because as tuples ($s_t$, $a_t$, $R_{t+1}$, $s_{t+1}$) are observed following $\pi_\theta$, we are guaranteed to end up in the limit in the stationary distribution $d^{\theta^*}$ of a local optima $\theta^*$. In the episodic setting with $\gamma=1$, the same remark holds. Nevertheless, it much less clear to the author in the infinite horizon discounted reward setting (with $\gamma<1$) how the policy gradient equations (\ref{eqac}) could be used rigorously, as they involve the $\gamma$ and $\gamma^2$-discounted visiting distribution. This subtlety is the focus of \cite{Thomas2014-ui} and was pointed out as well in \cite{Prashanth2016-fz} where they mention difficulties linked to the two policy gradient equations involving two distinct distributions: the $\gamma$ and $\gamma^2$ discounted visiting distributions, leading them to introduce new SF-based and SPSA-based algorithms.
\end{remark}

In the average reward case we use the notations of section \ref{AQL} to get algorithm \ref{cmvaca}, which is a modification of algorithm 2 in \cite{Prashanth2016-fz}. The three differences are i) the use of the chaotic TD error $\delta_{2,t}$ based on the Bellman equation in theorem 2, ii) the update of the distributional function $\widehat{R}_\psi$ as it is done in algorithm 3 and iii) the use of the policy gradient equation (\ref{eqac}) for the chaotic variance. In that case the value functions $V_\pi(s)$ and $V^{\mathbb{V}(\beta)}_\pi(s)$ are approximated by linear functions on lower dimensional spaces $\lambda_1^T\phi_1(s)$ and $\lambda_2^T\phi_2(s)$, respectively. The algorithm uses three timescales, so that parameters updated on the faster timescales can be considered as converged to their limiting value when considering updates on the slower timescales. 

In the episodic setting with $\gamma=1$, we obtain algorithm \ref{cmvace}, which is the episodic counterpart of algorithm \ref{cmvaca}.

It is to be noted that we can easily extend algorithm \ref{cmvace}, \ref{cmvaca} to the chaotic Sharpe ratio case, introduced in section \ref{CSR}. To do so, we use the theta update taken from section \ref{CSR}:
$$
\theta_{t+1}=\theta_t+ \alpha^{(1)}_n \left(\frac{\varphi_t\delta_{1,t}}{\sqrt{\delta_{2,t}}}  -\frac{1}{2}\frac{\delta_{1,t}\varphi_t\delta_{2,t}}{\delta_{2,t}^{3/2}}\right)
$$
with $\varphi_t:=\nabla \ln \pi_\theta(a_{t}|s_{t})$.

We do not discuss here the Actor-Critic extensions of the chaotic risk framework to the infinite horizon discounted reward framework in the CVaR case (\cite{Chow2018-eq}) and the mean-variance case (\cite{Prashanth2016-fz}), but they can be performed using similar techniques as used in the present section and in sections \ref{CSR}, \ref{CSR2}. In particular, the SF-based and SPSA-based techniques used in the above mentioned literature in the discounted reward framework go through with minor modifications as in algorithm \ref{cmvaca}, i) computing the chaotic TD error $\delta_{2,t}$ based on the Bellman equation in theorem 2 and ii) computing the update of the distributional function $\widehat{R}_\psi$ as it is done in algorithm 3.

\begin{algorithm}[htpb]
\caption{\textbf{Online CMV-Actor-Critic (episodic case)}}\label{cmvace}
\textbf{Input:} {Initial policy parameter $\theta_0$, initial critic parameters $\lambda_{j,0}$ ($j=1..2$), learning rates $\alpha^{(j)}_t$ ($j=1..3$), value function feature vectors $\phi_j$ ($j=1..2$), $\widehat{R}_\psi$ initialized to 0, Optional: experience replay table $\EEE$}\\
\textbf{Output:} {Approximation of the optimal policy parameter $\theta^*$}
\begin{algorithmic}[1]
	\While{$\theta$ not converged}
	    \State{observe tuple ($s_t$, $a_t$, $R_{t+1}$, $s_{t+1})$ following $\pi_\theta(\cdot|\cdot)$}
	    \State{\textbf{In tabular case:} use tuple to update $\widehat{R}_\psi$ with eq. (2); \textbf{else:} increment $\EEE$ with the tuple and update $\widehat{R}_\psi$ using algorithm 2.}
	    \State{\textbf{TD errors:}} \State{$\delta_{1,t}=R_{t+1}+ \lambda_{1,t}^T\phi_1(s_{t+1})-\lambda_{1,t}^T\phi_1(s_{t})$}
	    \State{$\delta_{2,t}=(R_{t+1}-\widehat{R}_\psi(s_t,a_t))^2+\lambda_{2,t}^T\phi_2(s_{t+1})-\lambda_{2,t}^T\phi_2(s_{t})$}
	    \State{\textbf{Critic Updates:}} \State{$\lambda_{1,t+1}=\lambda_{1,t}+\alpha^{(2)}_t \delta_{1,t}\phi_1(s_{t})$}
	    \State{$\lambda_{2,t+1}=\lambda_{2,t}+\alpha^{(2)}_t \delta_{2,t}\phi_2(s_{t})$}
	    \State{\textbf{Policy Update:}}
	    \State{$\theta_{t+1}=\theta_{t}+\alpha^{(1)}_t \nabla \ln \pi_\theta(a_{t}|s_{t})(\delta_{1,t}-\frac{\beta}{2}\delta_{2,t})$}
    \EndWhile
\end{algorithmic}
\end{algorithm}

\begin{algorithm}[htpb]
\caption{\textbf{Online CMV-Actor-Critic (average reward case)}}\label{cmvaca}
\textbf{Input:} {Initial policy parameter $\theta_0$, initial critic parameters $\lambda_{j,0}$ ($j=1..2$), learning rates $\alpha^{(j)}_t$ ($j=1..3$), value function feature vectors $\phi_j$ ($j=1..2$), $\widehat{R}_\psi$ initialized to 0, Optional: experience replay table $\EEE$.}\\
\textbf{Output:} {Approximation of the optimal policy parameter $\theta^*$}
\begin{algorithmic}[1]
	\While{$\theta$ not converged}
	    \State{observe tuple ($s_t$, $a_t$, $R_{t+1}$, $s_{t+1})$ following $\pi_\theta(\cdot|\cdot)$}
	    \State{\textbf{In tabular case:} use tuple to update $\widehat{R}_\psi$ with eq. (2); \textbf{else:} increment $\EEE$ with the tuple and update $\widehat{R}_\psi$ using algorithm 2.}
	    \State{\textbf{Average rewards:}} \State{$\rho_{t+1}=(1-\alpha^{(3)}_t)\rho_t+\alpha^{(3)}_t R_{t+1}$}
	    \State{$\sigma_{t+1}=(1-\alpha^{(3)}_t)\sigma_t+\alpha^{(3)}_t (R_{t+1}-\widehat{R}_\psi(s_t,a_t))^2$}
	    \State{\textbf{TD errors:}} \State{$\delta_{1,t}=R_{t+1}-\rho_{t+1}+\lambda_{1,t}^T\phi_1(s_{t+1})-\lambda_{1,t}^T\phi_1(s_{t})$}
	    \State{$\delta_{2,t}=(R_{t+1}-\widehat{R}_\psi(s_t,a_t))^2-\sigma_{t+1}+\lambda_{2,t}^T\phi_2(s_{t+1})-\lambda_{2,t}^T\phi_2(s_{t})$}
	    \State{\textbf{Critic Updates:}} \State{$\lambda_{1,t+1}=\lambda_{1,t}+\alpha^{(2)}_t \delta_{1,t}\phi_1(s_{t})$}
	    \State{$\lambda_{2,t+1}=\lambda_{2,t}+\alpha^{(2)}_t \delta_{2,t}\phi_2(s_{t})$}
	    \State{\textbf{Policy Update:}}
	    \State{$\theta_{t+1}=\theta_{t}+\alpha^{(1)}_t \nabla \ln \pi_\theta(a_{t}|s_{t})(\delta_{1,t}-\frac{\beta}{2}\delta_{2,t})$}
	   \EndWhile
\end{algorithmic}
\end{algorithm}

\section{Numerical values used in experiments of section 5 and additional experiments}
\label{secapplnum}
\subsection{Grid World - section 5.1}
We train the policy using $5 \cdot 10^5$ timesteps for each value of $\beta$ using an $\epsilon-$greedy exploration policy with associated probability to take a random action of $\epsilon=0.1$, a $Q$ table initialized to zero, a learning rate $\alpha_t=N(s_t,a_t)^{-0.5}$ (where $N(s_t,a_t)$ counts the number of visits to the state-action pair $(s_t,a_t)$), and set the probability of the robot going in a random direction to be $p_{error}=0.5$. We then run the trained policy for $10^5$ rollout steps in order to get the path heatmap and risk penalty heatmap displayed in figure 2. In the latter, the results are averaged over 25 NumPy RNG seeds (1001, 1003, 1006, 1008, 1009, 1010, 1015, 1018, 1021, 1022, 1025, 1028, 1029, 1031, 1033, 1035, 1037, 1038, 1039, 1040, 1041, 1043, 1047, 1048, 1049): for each such seed, we perform the training and obtain the rollouts, which are then averaged over seeds.

In figure \ref{fig-robot2} we display the path heatmaps similar as in figure 2 but for various values of $p_{error}$, representing the probability that the robot goes in a random direction. Figure \ref{fig-robot3} is similar to figure \ref{fig-robot2} but with a more severe negative reward of -50 instead of -20. As expected, in figure \ref{fig-robot2}, the higher $p_{error}$, the more reward uncertainty/stochasticity there is and the further away from the -20 reward the robot goes. The same observation goes for a lower reward -50 instead of -20 (figure \ref{fig-robot3}).

\begin{figure}[ht]
\vskip 0.1in
\begin{center} 
\centerline{\includegraphics[scale=0.4]{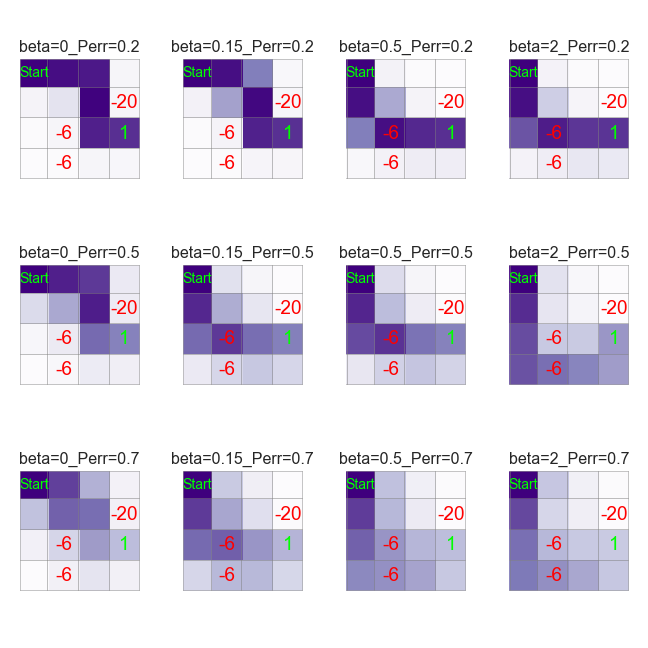}}
\caption{Path heatmaps over $10^5$ rollout steps with the learned policy for various $\beta$ and $p_{error}$}
\label{fig-robot2}
\end{center}
\vskip -0.1in
\end{figure}

\begin{figure}[ht]
\vskip 0.1in
\begin{center} 
\centerline{\includegraphics[scale=0.4]{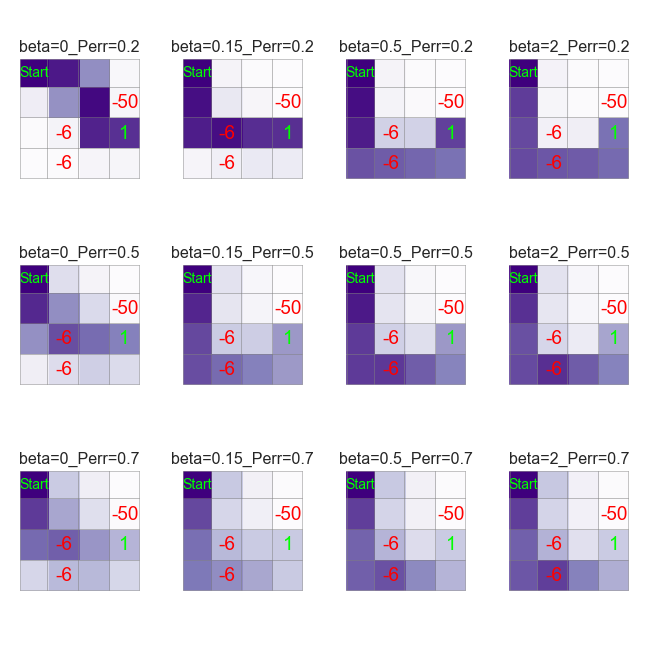}}
\caption{Path heatmaps over $10^5$ rollout steps with the learned policy for various $\beta$ and $p_{error}$, and more severe negative reward of -50 instead of -20.}
\label{fig-robot3}
\end{center}
\vskip -0.1in
\end{figure}

\subsection{Portfolio optimization - section 5.2}
As described in section 5.2, the action is defined as the quantities invested in the risk-free and risky assets $a_t:=(q^{RF}_{t},q^{R}_{t})$ with a total budget constraint $q^{RF}_{t}+q^{R}_{t} \leq q_{max}=5$ and $q^{RF}_{t},q^{R}_{t} \geq 0$. This yields a total of 21 possible actions $(0,0)$, $(1,0)$, $(0,1)$,..., $(4,1)$, $(1,4)$, $(0,5)$, $(5,0)$.

In order to train the risk-sensitive policy in section 5.2 according to the CMV-REINFORCE algorithm 3, we use a learning rate $\alpha=0.1$, batch size $B=10^4$ and perform $M=5 \cdot 10^3$ iterations on the policy parameter $\theta$. We use the Boltzmann exploration policy:
$$
\pi_\theta(a|s)=\frac{e^{\theta^T \phi(s,a)}}{\sum_{a' \in \mathbb{A}} e^{\theta^T \phi(s,a')}}
$$
where $\theta$ is of size $|\mathbb{S}|\cdot |\mathbb{A}|=3 \cdot 21=63$ and $\phi(s,a)$ is the vector with entries all zero except the entry corresponding to $(s,a)$ which is set to 1. For each value of $\beta$, we train the policy as described above and plot in figure 3, 4 the corresponding metrics for $5 \cdot 10^4$ rollout episodes of $T=20$ timesteps.

The values of the risk-free rate $\mu$ and volatility $\sigma$ in each state are reported in table \ref{tabm1}. The state transition matrices $P(s,a,s')$ depending on the action $a$ (via the quantity invested in the risky asset $q^{R}_{t}$) are displayed in table \ref{tabp1}. They have been designed such that the more we invest in the risky asset, the more likely we are to reach a higher volatility state.

\begin{table}[ht]
\caption{Section 5.2 Portfolio Optimization: risk-free rate $\mu$ and volatility $\sigma$}
\label{tabm1}
\vskip 0.1in
\begin{center}
\begin{small}
\begin{tabular}{|c|c|c|c|}
\toprule
\diagbox{Parameter}{State} & LowVol & MediumVol & HighVol \\ 
\midrule
$\mu$ & $0.2$ & $0.6$ & $1.$ \\  
\hline
$\sigma$ & $0.5$ & $1.$ & $1.5$ \\  
 \hline
\bottomrule
\end{tabular}
\end{small}
\end{center}
\vskip -0.1in
\end{table}

\begin{table}[ht]
\caption{Section 5.2 Portfolio Optimization: state transition matrix as a function of the chosen action (quantity $q^{R}_{t}$ invested in the risky asset)}
\label{tabp1}
\vskip 0.1in
\begin{center}
\begin{small}
\begin{tabular}{|c|c|c|c|}
\toprule
$q^{R}_{t}=5$ & LowVol & MediumVol & HighVol \\ 
\midrule
LowVol & $0.05$ & $0.25$ & $0.7$ \\  
\hline
MediumVol & $0.05$ & $0.25$ & $0.7$ \\   
 \hline
HighVol & $0.05$ & $0.25$ & $0.7$ \\ 
 \hline
\bottomrule
\end{tabular}
\begin{tabular}{|c|c|c|c|}
\toprule
$2<q^{R}_{t}<5$& LowVol & MediumVol & HighVol \\ 
\midrule
LowVol & $0.1$ & $0.45$ & $0.45$ \\  
\hline
MediumVol & $0.1$ & $0.45$ & $0.45$ \\    
 \hline
HighVol & $0.1$ & $0.45$ & $0.45$ \\   
 \hline
\bottomrule
\end{tabular}
\begin{tabular}{|c|c|c|c|}
\toprule
$0<q^{R}_{t}\leq 2$ & LowVol & MediumVol & HighVol \\ 
\midrule
LowVol & $1/3$ & $1/3$ & $1/3$ \\  
\hline
MediumVol & $1/3$ & $1/3$ & $1/3$ \\ 
 \hline
HighVol & $1/3$ & $1/3$ & $1/3$ \\  
 \hline
\bottomrule
\end{tabular}
\begin{tabular}{|c|c|c|c|}
\toprule
$q^{R}_{t}=0$ & LowVol & MediumVol & HighVol \\ 
\midrule
LowVol & $0.5$ & $0.45$ & $0.05$ \\  
\hline
MediumVol & $0.5$ & $0.45$ & $0.05$ \\   
 \hline
HighVol & $0.5$ & $0.45$ & $0.05$ \\
 \hline
\bottomrule
\end{tabular}
\end{small}
\end{center}
\vskip -0.1in
\end{table}

In figure 3 and 4, we compare CMV-REINFORCE (algorithm 3) with the mean-variance baseline. Denoting $\VV_{\pi_\theta}$ the variance operator, the gradient of the variance $\nabla_\theta \VV_{\pi_\theta}[\RR_0|s_0]$ needed in the mean-variance method is computed using the classical likelihood ratio technique. Indeed, $\VV_{\pi_\theta}[\RR_0|s_0]=\EE_{\pi_\theta}[\RR_0^2|s_0]-\left(\EE_{\pi_\theta}[\RR_0|s_0]\right)^2$, and hence:
$$
\nabla_\theta \VV_{\pi_\theta}[\RR_0|s_0]=\nabla_\theta \EE_{\pi_\theta}[\RR_0^2|s_0]-2 \EE_{\pi_\theta}[\RR_0|s_0] \nabla_\theta \EE_{\pi_\theta}[\RR_0|s_0]
$$
The likelihood ratio technique then yields (\cite{Tamar2012-ki}, lemma 4.2):
\begin{equation}
\begin{aligned}
\label{eqlt}
\nabla_\theta \VV_{\pi_\theta}[\RR_0|s_0]=\frac{1}{B} \sum_{b=1}^B (V_{b}-2\mu_J J_{b})\sum_{t=0}^{T_b-1} \nabla \ln \pi_\theta(a^{(b)}_{t}|s^{(b)}_{t}) 
\end{aligned}
\end{equation}
with:
$$
J_{b}:=\sum_{t=0}^{T_b-1} \gamma^{t}R^{(b)}_{t+1}, \hspace{2mm}
V_{b}:=J_b^2, \hspace{2mm}
\mu_J:= \frac{1}{B} \sum_{b=1}^B J_{b}
$$
Note that we additionally apply an optimal baseline $\ell^*$ (\cite{Necchi2016-ef}, section 4.4.1.1.) by replacing $V_{b}-2\mu_J J_{b}$ by $V_{b}-2\mu_J J_{b}-\ell^*$ in equation (\ref{eqlt}) with $\ell^*$ the vector which $k^{th}$ component is given by:
$$
\ell^*_k=\frac{\EE_{\pi_\theta}[(\RR_0^2-2\mu_J\RR_0)(\partial_{\theta_k}\ln \pi_\theta)^2]}{\EE_{\pi_\theta}[(\partial_{\theta_k}\ln \pi_\theta)^2]}
$$

\end{document}